\RequirePackage{scrlfile} 
\PreventPackageFromLoading{subfig}
\documentclass{article}
\pdfoutput=1
\usepackage{subcaption}
\usepackage{svg}
\usepackage{hyperref}

\usepackage{tabularx}
\usepackage{wrapfig}
\usepackage{cancel}
\usepackage{enumitem}
\PassOptionsToPackage{numbers, compress}{natbib}

\usepackage[final]{neurips_2020}
\usepackage{multirow}
\setcitestyle{numeric}
\usepackage[english]{babel}
\usepackage[utf8]{inputenc} % allow utf-8 input
\usepackage[T1]{fontenc}    % use 8-bit T1 fonts
\usepackage{booktabs}
\usepackage{tabularx}
\usepackage{times}
\usepackage{hyperref}
\usepackage{url}
\usepackage{harpoon}
\usepackage{amssymb}
\usepackage{amsmath}
\usepackage{graphicx}
\usepackage{sidecap}
\usepackage{courier}
\usepackage{color}
\usepackage{caption}
\usepackage{subcaption}
\usepackage{lscape}
\usepackage{tikz}
\usepackage[export]{adjustbox}

\usepackage[acronym,smallcaps,nowarn,section,nogroupskip,nonumberlist]{glossaries}
\usepackage[nameinlink,capitalise]{cleveref}

\usetikzlibrary{arrows,intersections, decorations.markings}

\usepackage{amsthm}
\usepackage{booktabs}
\usepackage{nicefrac}       % compact symbols for 1/2, etc.
\usepackage{microtype}      % microtypography
\usepackage{amsfonts}       % blackboard math symbols
\usepackage{harpoon}

\newcommand{\be}{\begin{eqnarray} \begin{aligned}}
\newcommand{\ee}{\end{aligned} \end{eqnarray} }
\newcommand{\benn}{\begin{eqnarray*} \begin{aligned}}
\newcommand{\eenn}{\end{aligned} \end{eqnarray*} }

\newcommand{\vx}{x}
\newcommand{\vz}{z}
\newcommand{\x}{x}
\newcommand{\z}{z}

\newcommand{\px}{{p_{\theta}(\vx)}}

\newcommand{\iq}{{I_{q}(X;Z)}}

\DeclareMathOperator*{\argmax}{arg\,max}
\DeclareMathOperator*{\argmin}{arg\,min}
\newcommand{\scriptveryshortarrow}[1][3pt]{{%
    \hbox{\rule[\scriptratio\dimexpr\fontdimen22\textfont2-.2pt\relax]
               {\scriptratio\dimexpr#1\relax}{\scriptratio\dimexpr.4pt\relax}}%
   \mkern-4mu\hbox{\let\f@size\sf@size\usefont{U}{lasy}{m}{n}\symbol{41}}}}
   
\glsdisablehyper{}
\newacronym{AIS}{ais}{annealed importance sampling}
\newacronym{BQ}{bq}{Bayesian Quadrature}
\newacronym{BA}{ba}{Blahut-Arimoto}
\newacronym{AUC}{auc}{area under the curve}
\newacronym{BAR}{bar}{Bennett's Acceptance Ratio}
\newacronym{JS}{js}{Jensen-Shannon}
\newacronym{CFT}{cft}{Crooks's Fluctuation Theorem}
\newacronym{ELBO}{elbo}{evidence lower bound}
\newacronym{EUBO}{eubo}{evidence upper bound}
\newacronym{HMC}{hmc}{Hamiltonian Monte Carlo}
\newacronym{IB}{ib}{Information Bottleneck}
\newacronym{IS}{is}{importance sampling}
\newacronym{IWAE}{iwae}{importance weighted autoencoder}
\newacronym{KL}{kl}{Kullback-Leibler}
\newacronym{MCMC}{mcmc}{Markov Chain Monte Carlo}
\newacronym{RD}{rd}{rate-distortion}
\newacronym{RWS}{rws}{reweighted wake-sleep}
\newacronym{SGD}{sgd}{stochastic gradient descent}
\newacronym{SNIS}{snis}{self-normalized importance sampling}
\newacronym{TI}{ti}{thermodynamic integration}
\newacronym{TVI}{tvi}{thermodynamic variational inference}
\newacronym{TVO}{tvo}{thermodynamic variational objective}
\newacronym{VAE}{vae}{variational autoencoder}
\newacronym{VI}{vi}{variational inference}
\newacronym{VIMCO}{vimco}{variational inference for Monte Carlo objectives}
\newacronym{WS}{ws}{wake-sleep}

\newcommand{\kld}{\textsc{kl} }
\def\KL{\mathrm{KL}}

\setlist[itemize]{leftmargin=*}

\begin{document}

% Theorem environment commands
\makeatletter
\def\thm@space@setup{%
  \thm@preskip=\parskip \thm@postskip=0pt
}
\makeatother

\newtheorem{theorem}{Theorem}[section]
\newtheorem*{theorem*}{Theorem}
\newtheorem{lemma}[theorem]{Lemma}
\newtheorem{proposition}[theorem]{Proposition}
\newtheorem{corollary}[theorem]{Corollary}

\newenvironment{definition}[1][Definition]{\begin{trivlist}
\item[\hskip \labelsep {\bfseries #1}]}{\end{trivlist}}
\newenvironment{example}[1][Example]{\begin{trivlist}
\item[\hskip \labelsep {\bfseries #1}]}{\end{trivlist}}
\newenvironment{remark}[1][Remark]{\begin{trivlist}
\item[\hskip \labelsep {\bfseries #1}]}{\end{trivlist}}
\newenvironment{theoremun}[1][Theorem]{\begin{trivlist}
\item[\hskip \labelsep {\bfseries #1}]}{\end{trivlist}}
 \setlength{\tabcolsep}{4pt}

%\icmltitlerunning{$\alpha$ Rate-Distortion}

%\title{Legendre Duality, Hypothesis Testing, and Rate-Distortion}
\title{Likelihood Ratio Exponential Families}
\author{
    Rob Brekelmans \\ %, Aram Galstyan, Greg Ver Steeg \\
    USC Information Sciences Institute \\
    %University of Southern California\\
    Marina del Rey, CA, USA \\
    \texttt{brekelma@usc.edu} 
    % \And
    % Vaden Masrani \\
    % University of British Columbia \\
    % Vancouver, CA \\
    % \texttt{vmasrani@cs.ubc.edu} 
    % \And
    % Alireza Makhzani \\
    % Vector Institute\\
    % %University of Toronto\\
    % Toronto, CA \\
    % \texttt{makhzani@vectorinstitute.ai}
    % \And
    % Greg Ver Steeg \\
    % USC Information Sciences Institute\\
    % %University of Southern California \\
    % Marina del Rey, CA, USA \\
    % \texttt{gregv@isi.edu}\\
    \And
    Frank Nielsen \\
    Sony CSL \\
     %Sony CSL, Ecole Polytechnique \\
    Tokyo, Japan \\
    \texttt{frank.nielsen@acm.org}
    %     \And
    % Alireza Makhzani \\
    % University of Toronto \\
    %  %Sony CSL, Ecole Polytechnique \\
    % Vector Institute \\
    % \texttt{frank.nielsen@acm.org}
        \And
    Aram Galstyan \\ %, \, Greg Ver Steeg \\
    USC Information Sciences Institute\\
    %University of Southern California \\
    Marina del Rey, CA, USA \\
    \texttt{galstyan@isi.edu}\\
    \And
    Greg Ver Steeg \\
    USC Information Sciences Institute\\
    %University of Southern California \\
    Marina del Rey, CA, USA \\
    \texttt{gregv@isi.edu}
    }
    
\author{
Rob Brekelmans$^1$,
Frank Nielsen$^2$,
Alireza Makhzani$^3$\\
\bf Aram Galstyan$^1$,
\bf Greg Ver Steeg$^1$ \\
$^1$USC Information Sciences Institute, $^2$Sony CSL, Tokyo \\ $^3$University of Toronto, Vector Institute \\
{\tt \{brekelma,galstyan,gregv\}@isi.edu} \\
{\tt makzhani@vectorinstitute.ai},  {\tt  frank.nielsen@acm.org }
}    
% \begin{document}
% {\let\newpage\relax\maketitle}
% %\maketitle 
% %\twocolumn{
% \input{sections/q}
% \input{sections/pf_q_min_div}
% \clearpage
%\input{sections/one_pager}
% %}
%
\maketitle 

\begin{abstract}
The exponential family is well known in machine learning and statistical physics as the maximum entropy distribution subject to a set of observed constraints \cite{jaynes1957information}, while the geometric mixture path is common in \textsc{mcmc} methods such as \gls{AIS} \cite{neal2001annealed}.   Linking these two ideas, recent work \cite{brekelmans2020all} has interpreted the geometric mixture path as an exponential family of distributions to analyse the \gls{TVO} \cite{masrani2019thermodynamic}.  

We extend these \textit{likelihood ratio exponential families} to include solutions to \gls{RD} optimization \cite{alemi2018fixing, cover2012elements}, the \gls{IB} method \cite{tishby1999information}, and recent rate-distortion-classification (\textsc{rdc}) approaches which combine \gls{RD} and \gls{IB} \cite{gao2020free, alemi2018therml}.  This  provides a common mathematical framework for understanding these methods via the conjugate duality of exponential families and hypothesis testing.  Further, we collect existing results \cite{Banerjee2005, grosse2013annealing, nielsen2013information, borade2006projection} to provide a variational representation of intermediate \gls{RD} or \gls{TVO} distributions as a minimizing an expectation of \textsc{kl} divergences.  This solution also corresponds to a size-power tradeoff using the likelihood ratio test and the Neyman Pearson lemma.   In \gls{TI} bounds \cite{ogata1989monte, gelman1998simulating} such as the \gls{TVO}, we identify the intermediate distribution whose expected sufficient statistics match the log partition function.
\end{abstract}

\section{Introduction}\label{sec:lr_families}
%\section{Likelihood Ratio Exponential Families}\label{sec:lr_families}
\paragraph{Likelihood Ratio Exponential Family} Following \citet{grunwald2007minimum} Ch. 19, or \citet{brekelmans2020all},  we consider the geometric mixture path between a base $\pi_0(\vz)$, and a target $\pi_1(\vz)$ or posterior $\pi_1(\vz|\vx)$, as an exponential family of distributions.
%In particular, we assume that the unnormalized target factorizes as $\tilde{\pi}_1(\vz) = \pi_0(\vz) \pi_1{a}(\vx|)$
We define the sufficient statistics  as the $\log$ likelihood ratio $\phi(\vz) = \log \pi_1(\vz) / \pi_0(\vz)$ \cite{brekelmans2020all}, although in practice it is convenient to consider unnormalized distributions such as $\pi_1(\vz) \propto \tilde{\pi}_1(\vz)$ or $\pi_1(\vz|\vx) \propto \tilde{\pi}_1(\vx,\vz)$ and adjust the normalization constant accordingly. 
Using a natural parameter $\beta$ and base measure $\pi_0(\vz)$, 
%we define the \textit{likelihood ratio exponential family}
%We can write this family $\pi_{\beta}$ using the natural parameter $\beta$ and base distribution $\pi_0$ as follows, 
\begin{align} %-d\textbf{}(\vx, \vz)
    \pi_{\beta}(\vz)&= \frac{1}{Z_{\beta}} \tilde{\pi}_0(\vz)^{1-\beta} \tilde{\pi}_1(\vz)^{\beta} \\[1.15ex]
    &= \tilde{\pi}_0(\vz) \exp \bigg\{ \, \beta  \cdot  \phi(\vz) \,   - \psi( \beta) \bigg\} \label{eq:lkd_ratio_fam} \\[1.15ex]
    \phantom{\pi_{\beta}} \text{where } \, \, \phi(\vz)   := \log \frac{\tilde{\pi}_1(\vz)}{\tilde{\pi}_0(\vz)} & \quad \quad  \psi(\beta) := \log Z_{\beta} =  \log \int \tilde{\pi}_0(\vz)^{1-\beta} \tilde{\pi}_1(\vz)^{\beta} d\vz \, . \label{eq:path_exp_family_general}
\end{align}

% \begin{align} %-d\textbf{}(\vx, \vz)
%     \pi_{\beta}(\vz)  = \tilde{\pi}_0(\vz) \exp \{ \beta  \cdot  \phi(\vz) \,   - \psi( \beta) \} = &\frac{1}{Z_{\beta}} \tilde{\pi}_0(\vz)^{1-\beta} \tilde{\pi}_1(\vz)^{\beta} \label{eq:lkd_ratio_fam}  \\
%     \phantom{\pi_{\beta}} \text{where } \, \, \phi(\vz)   := \log \frac{\tilde{\pi}_1(\vz)}{\tilde{\pi}_0(\vz)} \quad \quad \psi(\beta) := \log Z_{\beta} = & \log \int \tilde{\pi}_0(\vz)^{1-\beta} \tilde{\pi}_1(\vz)^{\beta} d\vz \, . \label{eq:path_exp_family_general}
% \end{align}
Before discussing examples in Sec. \ref{sec:examples}, we review background on conjugate duality in exponential families, which yields insights which are not evident from writing \cref{eq:lkd_ratio_fam} as a geometric mixture \cite{brekelmans2020all}.
\paragraph{Legendre Duality in Exponential Families}
Since the log partition function $\psi(\beta)$ of an exponential family is a strictly convex, analytic function of the natural parameters $\beta$, its gradient will be unique and may be used as a dual parameterization for $\pi_{\beta}$ \citep{wainwrightjordan}.  This diffeomorphism between the natural parameters $\beta = \{ \beta_j \}$ \footnote{We allow for multiple sufficient statistics, with $\beta \cdot \phi(\vz) = \sum_j \beta_j \cdot \phi_j(\vz)$ denoting the dot product.} and moment parameters, denoted $\eta = \{ \eta_j \}$, also defines the convex conjugate function $\psi^{*}(\eta)$, with %We discuss an example of how to construct a likelihood ratio family with multiple sufficient statistics in Sec. \eqref{sec:rdc}.
\begin{align}
\psi^{*}(\eta) = \sup_{\beta} \beta \cdot \eta - \psi(\beta) \qquad \implies \quad \eta_j &= \frac{\partial \psi}{\partial \beta_j}= \mathbb{E}_{\pi_{\beta}} [ \phi_j(\z)] \, \, \forall \, \, j \,. \label{eq:conjugate} 
% \eta_j (\beta) = \frac{\partial \psi}{\partial \beta_j} = \mathbb{E}_{\pi_{\beta}} [ \phi_j(\x, \z)] \qquad \forall 1\leq j\leq M \label{eq:logpartition_deriv} \, .
\end{align}
Using the Lebesgue or counting measure as $\pi_0(\vz)$, the conjugate $\psi^{*}(\eta)$ corresponds to the negative entropy of the maximum entropy solution $\pi_{\beta}(\vz)$ with observable constraint $\eta$ \citep{ wainwrightjordan, amari2016information}.  With a general base measure (see App. A), we have 
\begin{align}
\psi^{*}(\eta_{\beta}) = D_{\KL}[\pi_{\beta}(\z)||\pi_0(\z)] \label{eq:conjugate_kl}
\end{align}
Since the convex conjugate is an involution  $(\psi^{*})^{*} = \psi$ by the Moreau biconjugation theorem \cite{borwein2010convex}, we can obtain a similar  optimization for $\psi(\beta) = \sup_{\eta} \beta \cdot \eta - \psi^{*}(\eta)$.  This leads to the canonical expression for Legendre duality, when the two optimizations are in equilibrium, and $\beta$ and $\eta_{\beta}$ are in correspondence (see App. \ref{app:canonical})
\begin{align}
 \psi(\beta) + \psi^{*}(\eta_{\beta})  - \beta \cdot \eta_{\beta} = 0 \, . \label{eq:legendre}
\end{align}
% Finally, we can construct Bregman divergences from the convex functions $\psi(\beta)$ or $\psi^{*}(\eta)$.   Using \eqref{eq:path_exp_family_general} and \eqref{eq:legendre}, $D_{\psi}[\beta : \beta^{\prime}] := \psi(\beta) - \psi(\beta^{\prime}) - \langle \beta - \beta^{\prime} , \nabla \psi(\beta^{\prime}) \rangle = D_{\psi^{*}}[ \eta_{\beta^{\prime}}: \eta_{\beta} ]= D_{\KL}[ \pi_{\beta^{\prime}} || \pi_{\beta} ]$ \cite{nielsen2009statistical}.

%$D_{\psi}[\beta : \beta^{\prime}] = \psi(\beta) - \psi(\beta^{\prime}) - \langle \beta - \beta^{\prime} , \nabla \psi(\beta^{\prime}) \rangle $.

Using any convex function, we can obtain a Bregman divergence via the first order Taylor remainder.  For example, using $\psi(\beta)$ or $\psi^{*}(\eta)$, we have
\begin{equation*}
\begin{adjustbox}{max width=\textwidth}
%\begin{aligned}
  $D_{\psi}[\beta : \beta^{\prime}] = \psi(\beta) - \psi(\beta^{\prime}) - \langle \beta - \beta^{\prime}, \nabla \psi(\beta^{\prime}) \rangle \, \quad \,
%\end{aligned}\hspace*{.5cm}
%\begin{aligned}
  D_{\psi^{*}}[\eta : \eta^{\prime}] = \psi^{*}(\eta) - \psi^{*}(\eta^{\prime}) - \langle \eta - \eta^{\prime}, \nabla \psi^{*}(\eta^{\prime}) \rangle \, .$
%\end{aligned}
\end{adjustbox}
\end{equation*}
% In App. \ref{app:breg_kl}, we show that the Bregman divergences $D_{\psi}$, $D_{\psi^{*}$ associated with an exponential family correspond to the KL divergence
% \begin{align}
%      D_{\psi}[\beta : \beta^{\prime}] = D_{\KL}[ \pi_{\beta^{\prime}} || \pi_{\beta} ] = D_{\psi^{*}}[ \eta_{\beta^{\prime}}: \eta_{\beta} ] \, . \label{eq:divergences}
% \end{align}

% Applying the appropriate conjugacy relationships, for example $\psi^{*}(\eta_{\beta^{'}}) = \beta^{\prime} \cdot \eta_{\beta^{'}} - \psi(\beta^{\prime})$, 
% %or $\psi(\eta_{\beta}) = \beta \cdot \eta_{\beta} - \psi^{*}(\eta_{\beta})$, to 
% we can write the divergences, $D_{\psi}$ or $D_{\psi^{*}}$, in their canonical form

%Using the duality relationship \eqref{eq:legendre} and definition of the exponential family, 
With derivations in App. \ref{app:breg_kl}, we can see that the Bregman divergences $D_{\psi}, D_{\psi^{*}}$ are equivalent with the order of the arguments reversed, and correspond to a KL divergence % \cite{nielsen2009statistical}
\begin{align}
     D_{\psi}[\beta : \beta^{\prime}] = D_{\KL}[ \pi_{\beta^{\prime}} || \pi_{\beta} ] = D_{\psi^{*}}[ \eta_{\beta^{\prime}}: \eta_{\beta} ] \, . \label{eq:divergences}
\end{align}

\section{Examples}\label{sec:examples}
%We present several instances of the likelihood ratio exponential family from the generative modeling literature, primarily reframing existing methods in terms of Legendre duality.  

%While this may suggest future insights, we emphasize 
%Since these differ primarily in the choice of initial and target distributions, much of our analysis naturally translates across examples.  
%from the generative modeling literature 
% Summarized in Table \ref{tab:examples},

\paragraph{Thermodynamic Variational Objective}
In the \gls{VAE} setting, the \gls{TVO} \cite{masrani2019thermodynamic, brekelmans2020all} uses the
%The \gls{TVO} \cite{masrani2019thermodynamic, brekelmans2020all} considers the \gls{VAE} setting, using the
approximate posterior as the initial distribution $\tilde{\pi}_0 = q(\vz|\vx)$ and joint generative model as the unnormalized target $\tilde{\pi}_1 = p_{\theta}(\vx,\vz) \propto p_{\theta}(\vz|\vx)$.  
\begin{align}
    \pi_{\beta}(\vz|\vx) &=   q(\vz|\vx) \, \exp \bigg\{ \beta \cdot \log \frac{p_{\theta}(\vx,\vz)}{q(\vz|\vx)} - \psi(x; \beta)  \bigg\} \\
    &= \frac{1}{Z_{\beta}(\vx)} q(\vz|\vx)^{1-\beta} \,  p_{\theta}(\vx|\vz)^{\beta} \label{eq:q_star_rd}
\end{align}
with $\phi(\z) = \log \tilde{\pi}_1 /\tilde{\pi}_0 $.
\citet{masrani2019thermodynamic} use thermodynamic integration (\gls{TI}) \cite{ogata1989monte, gelman1998simulating} to express $\psi(\vx;1) = \log Z_{1}(\vx) = \log \px$ as an integral over the geometric path, %\eqref{eq:path_exp_family_general}, % mixture
\begin{align} %\limits_0^1
    \log Z_1(\vx) - \log Z_0(\vx) = \int_0^1  \frac{d}{d\beta} \log Z_{\beta} \, d\beta = \int_0^1 \mathbb{E}_{\pi_{\beta}} \big[ \phi(\vx, \vz) \big] \, d\beta \, . \label{eq:tvi_integral}
\end{align}
where we use the fact that the (partial) derivative of the log partition function equals the expected sufficient statistics in any exponential family \cite{wainwrightjordan}.  Since $\psi(\vx;\beta)$ is convex in $\beta$ for any $\vx$, the left- and right-Riemann sums will provide lower and upper bounds on the log marginal likelihood,
\begin{align}
   \hspace*{-.2cm} \sum \limits_{t=0}^{T-1} (\beta_{t+1} - \beta_{t}) \cdot \mathbb{E}_{\pi_{\beta_{t}}} \bigg[ \log \frac{\tilde{\pi}_1(\vx, \vz)}{\tilde{\pi}_0(\vz)} \bigg] \, \leq \, \log Z_1 \, \leq \,\sum \limits_{t=0}^{T-1} (\beta_{t+1} - \beta_{t}) \cdot \mathbb{E}_{\pi_{\beta_{t+1}}} \bigg[ \log \frac{\tilde{\pi}_1(\vx, \vz)}{\tilde{\pi}_0(\vz)} \bigg] %\label{eq:tvo_lb_ub3} \, ,
    % \log \frac{1}{K} \sum \limits_{k=1}^K \wkt := \mathcal{L}^{\gls{TVI}}_{\gls{IWAE}_{\textsc{lb}}}\label{eq:ti_iwae_lb}
\end{align}
% \begin{align}
%   \hspace*{-.2cm} \sum \limits_{t=0}^{T-1} (\beta_{t+1} - \beta_{t}) \cdot \mathbb{E}_{\pi_{\beta_{t}}} \big[ \log \frac{p_{\theta}(\vx,\vz)}{q(\vz|\vx)} \big] \, \leq \, \log Z_1 \, \leq \,\sum \limits_{t=0}^{T-1} (\beta_{t+1} - \beta_{t}) \cdot \mathbb{E}_{\pi_{\beta_{t+1}}} \big[ \log \frac{p_{\theta}(\vx,\vz)}{q(\vz|\vx)} \big] \label{eq:tvo_lb_ub2} \, .
%     % \log \frac{1}{K} \sum \limits_{k=1}^K \wkt := \mathcal{L}^{\gls{TVI}}_{\gls{IWAE}_{\textsc{lb}}}\label{eq:ti_iwae_lb}
% \end{align}
% \begin{align}
%   \sum \limits_{t=0}^{T-1} (\beta_{t+1} - \beta_{t}) \cdot \mathbb{E}_{\pi_{\beta_{t}}} \big[ \log \frac{\tilde{\pi}_1(\vx, \vz)}{\pi_0(\vz)} \big] \, \leq \, \log Z_1 \, \leq \,\sum \limits_{t=0}^{T-1} (\beta_{t+1} - \beta_{t}) \cdot \mathbb{E}_{\pi_{\beta_{t+1}}} \big[ \log \frac{\tilde{\pi}_1(\vx, \vz)}{\pi_0(\vz)} \big] \label{eq:tvo_lb_ub2} \, .
%     % \log \frac{1}{K} \sum \limits_{k=1}^K \wkt := \mathcal{L}^{\gls{TVI}}_{\gls{IWAE}_{\textsc{lb}}}\label{eq:ti_iwae_lb}
% \end{align}
 We derive novel insights on \gls{TVO} curve via hypothesis testing in Sec. \ref{sec:vrep}. Note that \gls{TI} bounds as in \eqref{eq:tvo_lb_ub2} may be constructed for any one-dimensional likelihood ratio  family with $\phi(z) = \log \tilde{\pi}_1/\tilde{\pi}_0$, such as in \gls{RD}.  However, more care would be required for multiple sufficient statistics as in \textsc{rdc} \cite{gao2020free, alemi2018therml}. %, including \gls{RD} as discussed below.

% The \gls{TVO} \cite{masrani2019thermodynamic, brekelmans2020all} constructs lower and upper bounds on the log marginal likelihood using \gls{TI} to express $\psi(\vx;1) = \log Z_{1} = \log \px$ as a one dimensional integral.  Noting that the partial derivative of the log partition function equals the expected sufficient statistics for any exponential family \cite{wainwrightjordan},
% \begin{align}
%     \log Z_1 - \log Z_0 = \int \limits_0^1 \frac{d}{d\beta} \log Z_{\beta} d\beta = \int \limits_0^1  \mathbb{E}_{\pi_{\beta}} \big[ \phi(\vz)   \big] \, d\beta \, . \label{eq:tvi_integral}
%     % IWAE ENERGY: \log \frac{1}{K} \sum \limits_{k=1}^K \wkt 
% \end{align} %\log \frac{\tilde{\pi}_1(\vz)}{\pi_0(\vz)}
% Since the log partition function is convex and 

\paragraph{Rate-Distortion}
Rate-distortion (\gls{RD}) optimization (\cite{alemi2018fixing, rose1998deterministic, tishby1999information,  cover2012elements} Ch. 13) formalizes the problem of lossy compression subject to a fidelity constraint.  As in \citet{alemi2018fixing}\cite{alemi2018therml}, we measure the rate using the \kld divergence to a fixed marginal distribution $\pi_0(\vz) = m(\vz)$, which upper bounds the mutual information in general.  The distortion function $d(\vx,\vz)$ measures the quality of a code $\vz$.  \gls{RD} optimization seeks the minimum-rate encoding which achieves a desired average distortion $D$,
\begin{align}
R(D) = \min \limits_{q(\vz|\vx)} D_{\KL}[q(\z|\x) || m(\z)] \quad \text{subj. to} \quad &\mathbb{E}_{q(\z | \x)} [ d(\x|\z) ] \leq D \, . \label{eq:rd_constrained}
%:= \eta_D \label{eq:fep} \\
%& \mathbb{E}_{q(\z| \x)} [ \log c(y|\z)] = -C := \eta_C \nonumber \, .
\end{align}
We restrict our attention to a reconstruction loss distortion $d(\vx,\vz) = -\log p_{\theta}(\vx|\vz)$ as in \cite{alemi2018fixing}.  Introducing $\beta$ to enforce the constraint, we obtain the unconstrained Lagrangian
\begin{align}
    \max_{\beta} \min_{q(\vz|\vx)} D_{\KL}[q(\z|\x) || m(\z)] - \beta \big( \, \mathbb{E}_{q(\z | \x)} [ d(\x,\z) ] - D \, \big) 
\end{align}
whose solution, for a given $m(\vz)$, has an exponential family form with $\phi(\vx,\vz) = -d(\vx,\vz)$ (e.g. \cite{tishby1999information})
\begin{align}
    \pi_{\beta}(\vz|\vx) &=  m(\vz) \, \exp \{-\beta \cdot d(\vx,\vz) - \psi(x; \beta) \} \\
    &= \frac{1}{Z_{\beta}(\vx)} m(\vz) \,  p_{\theta}(\vx|\vz)^{\beta} \label{eq:q_star_rd}
\end{align}
From the likelihood ratio perspective, we can choose $\pi_0(\vz) = m(\vz)$ and $\tilde{\pi}_1(\vx,\vz) = p_{\theta}(\vx|\vz) m(\vz) \propto p_{\theta}(\vz|\vx)$. Absorbing the factor of $\px$ into the normalizer $Z_{\beta}(\vx)$, we obtain the sufficient statistics
\begin{align}
    \phi(\vx,\vz) = \log \frac{\tilde{\pi}_1(\vx, \vz)}{\tilde{\pi}_0(\vz)} = \log \frac{p_{\theta}(\vx|\vz) m(\vz)}{m(\vz)} = \log p_{\theta}(\vx|\vz) = -d(\vx,\vz) \, , \label{eq:phi_d}
\end{align}%$\phi(\vx,\vz) = -d(\vx,\vz) = \log p_{\theta}(\vx|\vz)$
so that the solution in \eqref{eq:q_star_rd} matches $\pi_{\beta}(\vz|\vx)$ in the likelihood ratio family induced by \eqref{eq:phi_d}.    The Lagrange multiplier $\beta$ is chosen to enforce the distortion constraint $D$.  Since $\phi(\vx,\vz) = -d(\vx,\vz)$,  simply translates to seeking  moment parameters such that $\eta_{\beta} = -D$.  At this optimal solution, $R(D)$ matches the conjugate function $\psi^{*}(\eta)$ in \eqref{eq:conjugate_kl}, with
\begin{align}
    R(D) = \psi^{*}(\eta) &=  D_{\KL}[\pi_{\beta}(\vz|\vx) || m(\vz)] \\
    &  = \beta \cdot \eta - \psi(\beta) \label{eq:rd_conjugate0} \\
    & = -\beta \, D - \log Z_{\beta}(\vx) \, . \label{eq:rd_conjugate}
\end{align}
\citet{huangevaluating} use the expression in \eqref{eq:rd_conjugate} to estimate the \gls{RD} curve using \gls{AIS} \cite{neal2001annealed}.   Finally, from the conjugate optimization $\psi(\beta) = \sup_{\eta} \beta \cdot \eta - \psi^{*}(\eta)$, we obtain the familiar interpretation of the Lagrange multiplier as measuring the slope of the rate-distortion curve 
\begin{align}
\beta = \dfrac{d\psi^{*}}{d\eta} = -\dfrac{dR}{dD} \, .
\end{align}

\paragraph{Information Bottleneck and RDC}% and Rate-Distortion-Classification}
When defining `relevant information' via a random variable such as a label $y$, the Information Bottleneck (\gls{IB}) method \cite{tishby1999information, achille, alemi2016deep} simplifies to an \gls{RD} problem with a learned classifier providing the distortion function $ c(y, \z) = -\log p_{\theta}(y|\vz) $ (\cite{tishby1999information} or App.\ref{app:ib_pf}).  
\begin{align}
\min \limits_{q(\vz|\vx)} D_{\KL}[q(\z|\x) || m(\z)] \quad \text{subj. to} \quad & \mathbb{E}_{q(\z| \x)} [ c(y, \z)] \leq C  \label{eq:ib} 
\end{align}
Recent work \cite{gao2020free,alemi2018therml} considers `\textsc{rdc}' optimization using both reconstruction and classification loss, 
%loss, which we denote as `\textsc{rdc}' %, for applications in transfer learning %an optimization which we reference as \gls{RDC}, 
\begin{align}
\min \limits_{q(\vz|\vx)} D_{\KL}[q(\z|\x) || m(\z)] \quad \text{subj. to} \quad &\mathbb{E}_{q(\z | \x)} [ d(\x,\z) ] \leq D \, \, , \, \, \mathbb{E}_{q(\z| \x)} [ c(y, \z)] \leq C  \label{eq:fep1} 
\end{align}
In this case, we may consider two sufficient statistics in our likelihood ratio exponential family.  Similarly to multivariate \gls{IB} \cite{Slonim2006, elidan2002information}, we use an unnormalized target which factorizes as $\tilde{\pi}_1(\vx,y,\vz) = p_{\theta}(\vx|\vz) p_{\theta}(y| \vz)m(\vz)$, and consider the likelihood ratio sufficient statistics
%\begin{equation*}
%\begin{adjustbox}{max width=\textwidth}
\begin{align}
    \phi_d(\vx, \vz) &= \log \frac{\pi_1(\vz|\vx)}{\pi_0(\vz)} = \log \frac{ p_{\theta}(\vx|\vz)}{\px}  \propto \log p_{\theta}(\vx|\vz) \propto -d(\vx, \vz) \\
%     \nonumber %\label{eq:rdc_phi}
% \end{aligned}\hspace*{.5cm}
% \begin{aligned}
      \phi_c(y, \vz) &= \log \frac{\pi_1(\vz|y)}{\pi_0(\vz)} = \log \frac{ p_{\theta}(y| \vz)}{p(y)} \propto \log p_{\theta}(y|\vz) = -c(y, \vz) \nonumber
\end{align}
%\end{adjustbox}
%\end{equation*}
where we again absorb $\px$ and $p(y)$ into the normalization.
%, using $\phi_d(\vx,\vz) = -d(\vx, \vz)$ as in \eqref{eq:phi_d}.
Introducing Lagrange multipliers $\beta = \{\beta_D, \beta_C \}$ to enforce $\eta_{d}(\beta) = -D$, $ \eta_{c}(\beta)=-C$ at optimality, we obtain the solution to \cref{eq:fep1} as a geometric mixture \cite{gao2020free, alemi2018therml} belonging to the likelihood ratio family with $\phi = \{\phi_d, \phi_c\}$ % and  %\eqref{eq:rdc_phi}
\begin{align}
%\begin{adjustbox}{max width=\textwidth}
    \pi_{\beta}(\vz|\vx,y) &=  m(\z) \exp \big\{ \beta_D \cdot \phi_d(\vx,\vz) + \beta_C \cdot \phi_c(y,\vz) - \psi(\x, y; \beta) \big\} 
    \label{eq:rdc_sol} \\
    &= \frac{1}{Z_{\beta}(\x,y)} \, m(\z) \, p_{\theta}(\vx|\vz)^{\beta_D} \,p_{\theta}(y|\vz)^{\beta_C}   \label{eq:solution_ib}
    %m_0(\z) \exp \{ \lambda \cdot \log d(\x|\z) + \gamma \cdot \log c(y|\z) - \psi(\x; \lambda, \gamma)
%\end{adjustbox}
\end{align}
%where $\beta_D, \beta_C$ reflect the optimal set of Lagrange multipliers $\beta$ enforcing $D = -\eta_{d}(\beta)$, $C = -\eta_{c}(\beta)$. %and \eqref{eq:solution_ib} would reduce to \gls{IB} for $\beta_D = 0$.
%match \cite{gao2020free}, and the Lagrange multiplier $\gamma$ enforces the constraint on the classification loss.  
With applications in transfer learning, \citet{gao2020free} seek to evolve the model parameters $\theta$ and approximate posterior $q(\vz|\vx)$ along the `equilibrium surface' of optimal solutions to \cref{eq:fep1}.  We interpret their free energy $F(\beta_D, \beta_C)$ as the negative log partition function $-\psi(\beta_D, \beta_C)$,  where $\beta_D, \beta_C$ are analogous to the \textit{intensive} variables of a physical system \cite{alemi2018therml}.  Written using the conjugate optimization \eqref{eq:conjugate}, we seek $\{\theta$, $q(\vz|\vx)\}$ that yield the appropriate distortion and classification loss %$\eta_D, \eta_C$
\begin{align}
  -F(\beta_D, \beta_C)=  \psi(\beta_D, \beta_C) = \sup_{\eta_d,\eta_c}  \beta_D \, \eta_d + \beta_C \, \eta_c - \psi^{*}(\eta_d, \eta_c)  \label{eq:fe}
\end{align}
Similarly, for given \textit{extensive} variables $\eta_D, \eta_C$, the optimal rate $R(D,C)$ corresponds to $\psi^{*}(\eta_D, \eta_C)$ 
%Similarly, for given moment constraints $\eta_D, \eta_C$, or \textit{extensive variables}, the optimal rate $R(D,C)$ corresponds to the conjugate $ \psi^{*}(\eta_D, \eta_C)$   %This fact ahs  which has not been explicitly observed in \cite{alemi2018therml, gao2020free}. 
\begin{align}
R(D,C) = \psi^{*}(\eta_D, \eta_C) = \sup \limits_{\beta_d, \beta_c } - \beta_d \, D - \beta_c \, C - \psi(\beta_d, \beta_c) \, ,
\label{eq:leg} 
\end{align}
At optimality on the `equilibrium surface' \cite{gao2020free}, we obtain equality in the expression for Legendre duality \eqref{eq:legendre}.
In other words, for the current decoder and classifier parameters $\theta$, the encoder $q(\vz|\vx)$ matches $\pi_{\beta}(\vz|\vx)$ in the likelihood ratio family \eqref{eq:rdc_sol}, with $\beta = \{ \beta_D, \beta_C \}$.  This distribution fulfills the constraints $\eta_{\beta} = \{ \eta_D, \eta_C \} = \{ -D, -C \}$, so that 
% $q(\vz|\vx) = \pi_{\beta}(\vz|\vx)$ in the likelihood ratio family \eqref{eq:rdc_sol}.  This distribution fulfills the constraints $\eta_{\beta} = \{ \eta_D, \eta_C \} = \{ -D, -C \}$ for $\beta = \{ \beta_D, \beta_C \}$ and  the current decoder and classifier parameters $\theta$, and corresponds to equality in the Legendre duality expression \eqref{eq:legendre}
\begin{align}
  \psi(\beta_D, \beta_C) + \psi^{*}(\eta_D, \eta_C) - \beta_D \, \eta_D - \beta_C \, \eta_C = 0 \, . \label{eq:legendre_dc}
\end{align}

%and corresponds to equality in the Legendre duality expression \eqref{eq:legendre}

This expression \eqref{eq:legendre_dc} also translates to the `first law of learning' from \citet{alemi2018therml}, when $\psi(\beta_D, \beta_C)$ is considered a fixed quantity for given a choice of $\beta_D, \beta_C$.% is considered as a fixed quantity.

\section{Variational Representations and Hypothesis Testing}\label{sec:vrep}
\citet{grosse2013annealing} note that any distribution along the geometric mixture path can be given a variational representation as the solution to an expected \kld divergence minimization
\begin{align}
    \pi_{\beta}(\z) &= \argmin \limits_{r(\z)} (1-\beta) \, D_{\KL}[ r(\z) || \pi_{0}(\z)] + \, \beta \, D_{\KL}[  r(\z) || \pi_{1}(\z)] \, \label{eq:vrep1}
\end{align}%either 
We proceed to interpret \cref{eq:vrep1} as a Bregman information (or gap in Jensen's inequality) \cite{Banerjee2005}, and as describing an optimal decision rule for hypothesis testing using the Neyman Pearson lemma.  We restrict our attention to a one-dimensional likelihood ratio family, as in \gls{TVO}, \gls{RD}, or \gls{IB}, in this section.

%\subsection{Bregman Information}
\paragraph{Bregman Information}
\citet{Banerjee2005} define the \textit{Bregman information} as the minimum expected Bregman divergence to a representative point in the second argument.  Regardless of the Bregman generator, the optimal representative corresponds to the mean over the input arguments.  
Since $D_{\KL}[ r_{\beta}(\z) || \pi_{0}(\z)] = D_{\psi}[0 : \beta_r]$ when optimizing over $r_{\beta}(\vz)$ in the exponential family, we can rewrite \cref{eq:vrep1} as
\begin{align}
    \beta &= \argmin \limits_{\beta_r} \, (1-\beta) \, D_{\psi}[ 0 : \beta_r ] + \, \beta \, D_{\psi}[ 1 : \beta_r] \, \,  = \, (1- \beta) \cdot 0 + \beta \cdot 1  \label{eq:vrep_breg} 
    % \pi_{\eta_t}(\z) &= \argmin \limits_{r(\z)} (1-t) \, D_{\KL}[ \pi_{\beta_0}(\z) ||  r(\z)]  + \, t \, D_{\KL}[ \pi_{\beta_1}(\z) || r(\z)] \label{eq:vrep2} \\
    % & \text{where } \eta_t = (1-t) \cdot \eta_0 + t \cdot \eta_1 \nonumber
\end{align}
At this optimum, the expected \kld divergence \eqref{eq:vrep_breg} can be written as a gap in Jensen's inequality for the convex function $\psi(\beta)$ \cite{Banerjee2005},
% At this optimal representative distribution $\pi_{\beta_t}$, the expected \kld divergence \eqref{eq:vrep_breg} can be written as a gap in Jensen's inequality for the convex function $\psi(\beta)$ \cite{Banerjee2005}, or, as shown by  \citet{nielsen2011R\'enyi} or App. \ref{app:R\'enyi}, as a R\'enyi divergence with order $t$
\begin{align}
  \mathcal{J}_{\psi, \{ 1-\beta, \beta \}, \{ 0,1 \} } &= (1-\beta) \, D_{\psi}[ 0 : \beta ] + \, t \, D_{\psi}[ 1 : \beta ] \\
 &= (1-\beta) \, \psi(0) + \beta \, \psi(1) - \psi \big( \beta \big) \label{eq:jg00}  %\nonumber %\label{eq:jg_R\'enyi}
\end{align}
We visualize this gap in Jensen's inequality in \cref{fig:chernoff_logzb}.   \citet{nielsen2019abstract, nielsen2010family} utilize $\mathcal{J}_{\psi}$ to construct additional divergence measures, which \citet{deasy2020constraining} explore in the context of variational autoencoders.
%as alternative information measures

As shown in \cite{nielsen2011renyi} or App. \ref{app:renyi_jg}, we can also view $\mathcal{J}_{\psi}$, or the expected \kld divergence \eqref{eq:vrep1}, as a R\'enyi divergence with order $\beta$
\begin{align}
\mathcal{J}_{\psi, \{ 1-\beta, \beta \}, \{ 0,1 \} }  &= (1-\beta) \, D_{\KL}[ \pi_{\beta}(\z) || \pi_{0}(\z)] + \, \beta \, D_{\KL}[ \pi_{\beta}(\z) || \pi_{1}(\z)]  \\
&= (1-\beta) \, D_{\beta}[\pi_{1}(\z) : \pi_{0}(\z) ] \nonumber \, .
\end{align}
\citet{grunwald2007minimum} and \citet{harremoes2006interpretations} provide additional coding interpretations of the Rényi divergence.  

\paragraph{Neyman Pearson Lemma}
%\subsection{Neyman Pearson Lemma}\label{sec:neyman_pearson}
Suppose we have access to $n$ i.i.d. observations from an unknown distribution $r(\vz)$, and are interested in testing the hypotheses that either $H_0: r(\vz) = \pi_0(\vz)$ or $H_1: r(\vz) = \pi_1(\z)$.  The Neyman-Pearson lemma states that the likelihood ratio test is optimal, in the sense that, for any other decision region with type-1 error $\mathrm{Pr}(e_1) = R$, then the type-2 error is no better than that of the likelihood ratio test (\cite{cover2012elements} Ch. 11, \cite{borade2006projection})  \footnote{While the Neyman-Pearson lemma is often obtained via the discrete method of types \cite{cover2012elements}, \citet{csiszar1998method} gives a derivation for the continuous setting.}.  The decision rule is given by %with threshold $\eta$ 
\begin{align}
A_{n}(\pi_1; \eta)  = \bigg \{ \, z_{1:n} \, \, \bigg | \, \, \, \frac{1}{n} \sum \limits_{i=1}^n \log \frac{\pi_1(z_i)}{\pi_0(z_i)} \geq \eta  \, \bigg \}
\end{align}
for some threshold $\eta$.  Let a type-1 error occur when $n$ i.i.d. draws $\{ z_i \}_{i=1}^{N}$ from $\pi_0(\vz)$ will yield empirical expectations exceeding the threshold $\eta$.  Sanov's Theorem and large deviation theory (\cite{cover2012elements} Ch. 11, \cite{csiszar2004information}) states that the asymptotic error exponent corresponds to a \kld divergence 
\begin{align}
\hspace*{-.2cm}	\lim \limits_{n \rightarrow \infty} \frac{1}{n} \, \mathrm{Pr}(e_1) &\rightarrow \exp \big\{ -D_{\KL}[r^{*}(\z) || \pi_0(\z) ] \big\}  \, , \, \\[1.25ex]
\text{ where  } r^{*}(\vz) &= \min \limits_{r(\vz) \in \mathcal{M}_{\eta}} D_{\KL}[r(\vz)|| \pi_0(\z)] \, \label{eq:sanov0}
\end{align}
%where  Thus, the probability of error is governed by the `closest' distribution to $\pi_0(\vz)$ in \kld divergence within the feasible set of distributions satisfying the expectation constraint $\eta$.
The feasible set $\mathcal{M}_{\eta} := \{r(\vz) \, | \, \mathbb{E}_{r} \log \frac{\pi_1(\vz)}{\pi_0(\vz)} = \eta \}$ reflects a expectation constraint corresponding to a given decision threshold, and the error exponent is obtained by minimizing the divergence subject to this constraint.  With $\psi^{*}(\eta) = D_{\KL}[\pi_{\beta_{\eta}}(\vz)||\pi_0(\vz)]$ as in \cref{eq:conjugate_kl}, this exactly matches the conjugate or maximum entropy optimization for a given expected sufficient statistic $\eta$, and thus $r^{*}(\vz)$ lies within the likelihood ratio exponential family,
\begin{align}
    r^{*}(\vz) = \pi_0(\vz) \exp \big\{ \beta_{\eta} \cdot \log \frac{\pi_1(\vz)}{\pi_0(\vz)} - \psi(\beta) \big\}
\end{align}
As shown in Fig. \ref{fig:sanov}, Sanov's Theorem implies a similar expression for the asymptotic type-2 error, when draws from $\pi_1(\vz)$ achieve a \textit{lower} expected likelihood ratio than $\eta$.  Expressing the conditions of the Neyman Pearson lemma using these asymptotic error probabilities, we can write
%where we can treat $r(\vz)$ as the discrete `type' distribution, corresponding to the proportion $n_i / n$ of $n$ draws taking value $\vz = \vz_i$ \cite{cover2012elements}.  }
%\footnote{Note that the Neyman Pearson lemma is often derived using the method of types, where we can treat $r(\vz)$ as the discrete `type' distribution, corresponding to the proportion $n_i / n$ of $n$ draws taking value $\vz = \vz_i$ \cite{cover2012elements}.  The continuous case is treated by \citet{csiszar1998method}.}, we can write
\begin{align}
    \mathrm{Pr}(e_2) = \min \limits_{r(\vz)} D_{\KL}[r(\vz) || \pi_1(\vz)]  \quad \text{subj. to} \quad  D_{\KL}[r(\vz) || \pi_0(\vz)] = R \, 
\end{align}
Using a Lagrange multiplier $\lambda = \frac{1-\beta}{\beta}$ to enforce the constraint, we obtain the variational form \eqref{eq:vrep1} %to enforce the constraint
\begin{align}
    \frac{1}{\beta} \, \mathrm{Pr}(e_2) =  \min \limits_{r(\vz)} (1-\beta) \, D_{\KL}[r(\vz) || \pi_0(\vz)] + \beta \, D_{\KL}[r(\vz) || \pi_1(\vz)]
\end{align}
Thus, any distribution in our likelihood ratio exponential family corresponds to a likelihood ratio test with decision threshold $\eta$, which is optimal for a type-1 error region of size $\psi^{*}(\eta) = R$.

% \subsection{Sanov's Theorem: } \label{sec:sanov} 
% An alternative view of the exponential family is given by Sanov's Theorem and large deviation theory ( \cite{cover2012elements} Ch. 11, \cite{csiszar1984sanov,csiszar2004information}).  Let $M_{\pi_0, \eta}$ denote the event that i.i.d. draws $\{ z_i \}_{i=1}^{N}$ from $\pi_0(\vz)$ will yield empirical expectations given by $\eta$, with $\frac{1}{N} \sum_i \phi_j(z_i) = \eta_j$ $\forall j$.  Sanov's Theorem states that
% \begin{align}
% \hspace*{-.2cm}	\lim \limits_{n \rightarrow \infty} \frac{1}{n} Pr(M_{\pi_0, \eta}) \rightarrow \exp \{ -\kl[r^{*}(\z) || \pi_0(\z) ] \}  \, \, \text{ where } r^{*}(\vz) = \min \limits_{r(\vz) \in \mathcal{M}_{\eta}} \kl[r(\vz)|| \pi_0(\z)] \label{eq:sanov0}
% \end{align}
% where $\mathcal{M}_{\eta}$ is defined as the feasible set of distributions satisfying the expectation constraints.  

% %Thus, the asymptotic probability  $M_{\pi_0, \eta}$ is governed by $r^{*}(z)$, the `closest' distribution to $\pi_0$ within the feasible set \cite{eguchi2006interpreting, cover2012elements}.   However, we have already seen that the solution will correspond to a distribution in an exponential family as the solution to a \kld divergence minimization subject to a 
% %This solution $r^{*}$ will correspond with a distribution $\pi_{\beta_{\eta}}$ and the supremum of the conjugate optimization \eqref{eq:conjugate}-\eqref{eq:legendre_sol}.

\begin{figure*}[t]
%\begin{wrapfigure}{r}{0.5\textwidth} 
\begin{minipage}{\textwidth}
%\begin{minipage}{.45\textwidth}
% \vspace*{.2cm}
\centering
%\begin{subfigure}{.475\textwidth}
\includegraphics[width=.485\textwidth]{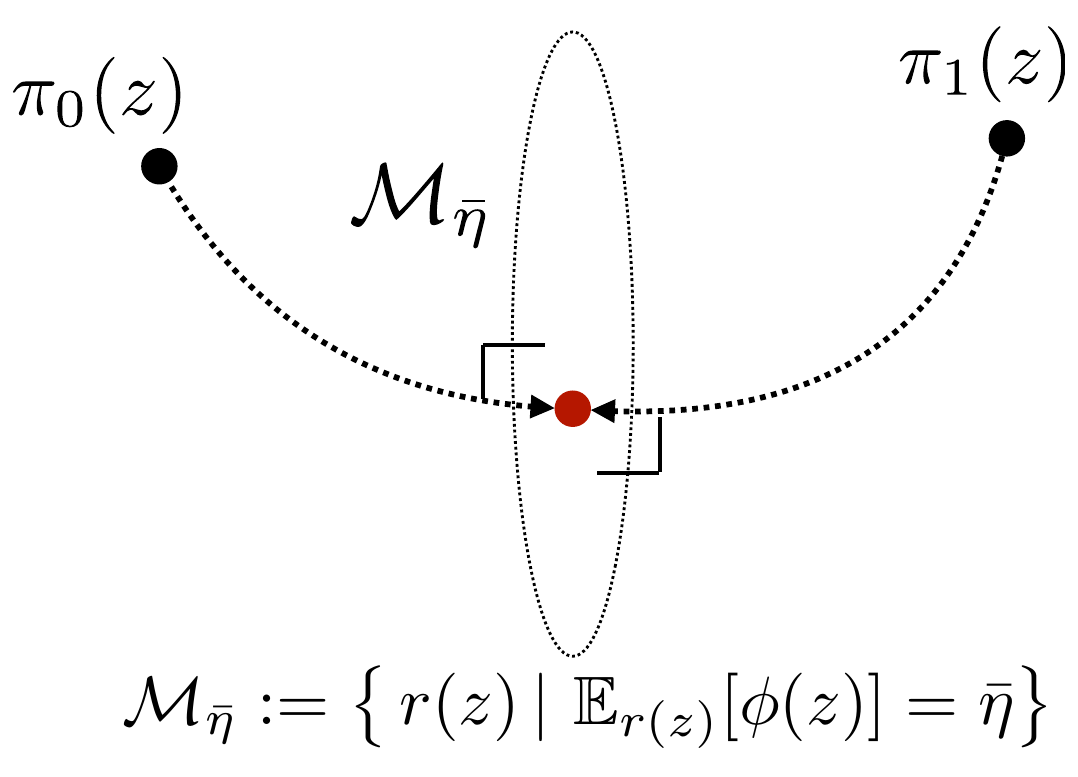}
%\vspace*{.6cm}
%\vspace*{.2cm}
\caption{Sanov's Theorem.  Each exponential arc intersects the manifold $\mathcal{M}_{\bar{\eta}}$ with a right angle, due to a generalized Pythagorean Theorem for the \textsc{kl} divergence  (see Sec. 1.6, 2.8 of \cite{amari2016information}). } \label{fig:sanov} 
%to the \textsc{KL} divergence.}  \label{fig:sanov} 
%:  For a given acceptance threshold $\bar{\eta}$, which defines an optimal tradeoff between size of the type 1 and type 2 error regions in the Neyman-Pearson lemma, we can define a manifold of distributions $\mathcal{M}_{\bar{\eta}}$ satisfying $\mathbb{E}_r[\log \pi_1 / \pi_0] = \bar{\eta}$.   Using Sanov's Theorem, we can see that the minimizers of both $D_{\KL}[r \, ||\pi_1]$ and $D_{\KL}[r \, ||\pi_0]$ for $r \in \mathcal{M}_{\bar{\eta}}$ will coincide, with $r^{*}(\vz|\vx) = \pi_{\beta}(\vz|\vx) \propto \pi_0^{1-\beta} \, \pi_1^{\beta}$ for $\beta$ such that $\eta_{\beta} = \mathbb{E}_{\pi_{\beta}}[\phi] = \bar{\eta}$.
\end{minipage}
\vspace{-10pt}
\end{figure*}
%\end{wrapfigure}

\begin{figure*}[t]
\begin{minipage}{.525\textwidth}
%\begin{minipage}{.55\textwidth}
\vspace*{-.2cm}
\centering
\includegraphics[width=.995\textwidth, trim=0 0 0 0, clip]{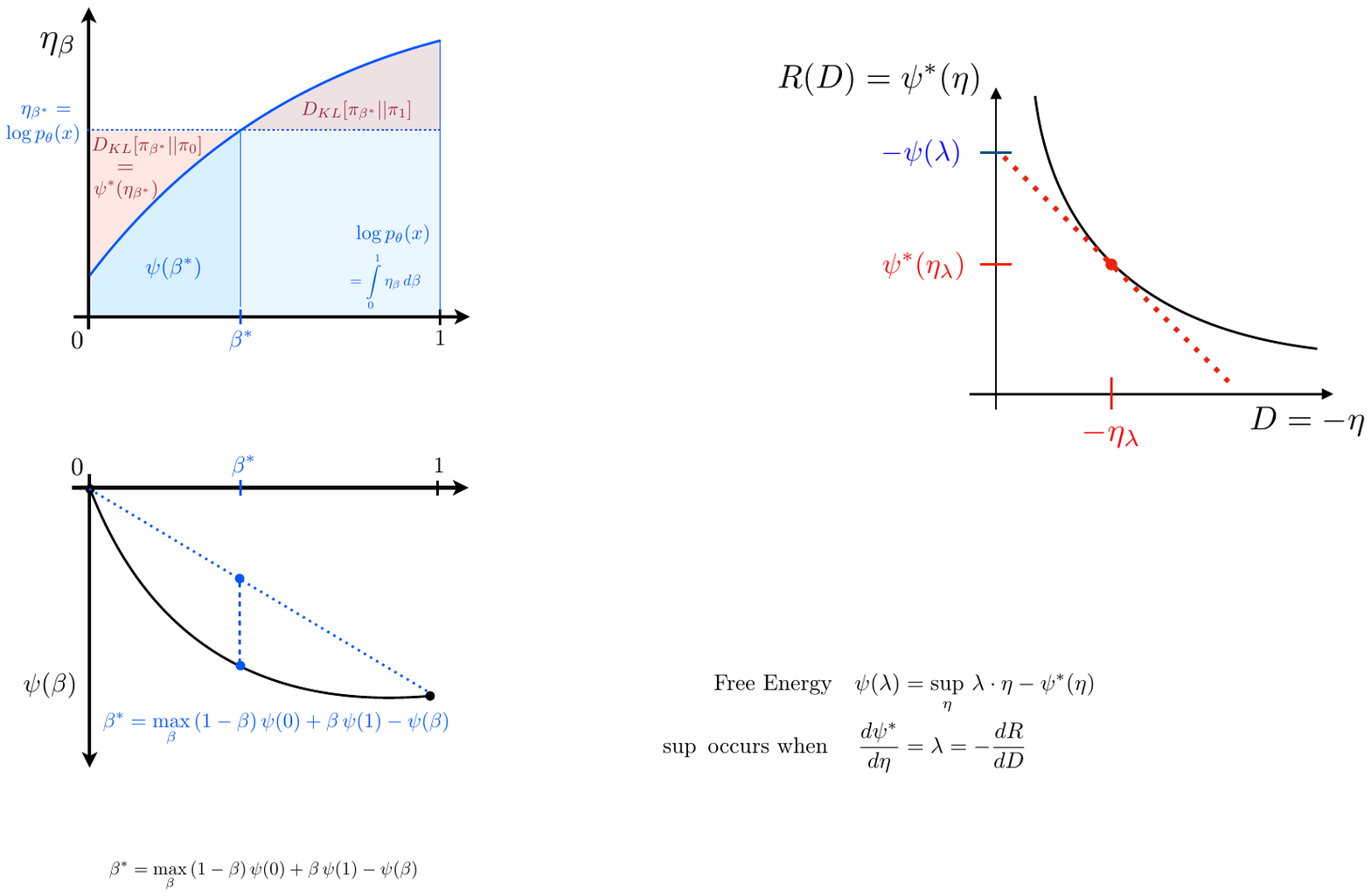}
\vspace*{.05cm}
\caption{Chernoff point on $\eta_{\beta} = \nabla \psi(\beta)$.} \label{fig:chernoff}
\end{minipage}\hspace*{.01\textwidth}
\begin{minipage}{.46\textwidth}
%\vspace*{-.075cm}
%\hspace*{-.11\textwidth}
%\centering
\includegraphics[width=1\textwidth, trim=2 0 0 0, clip]{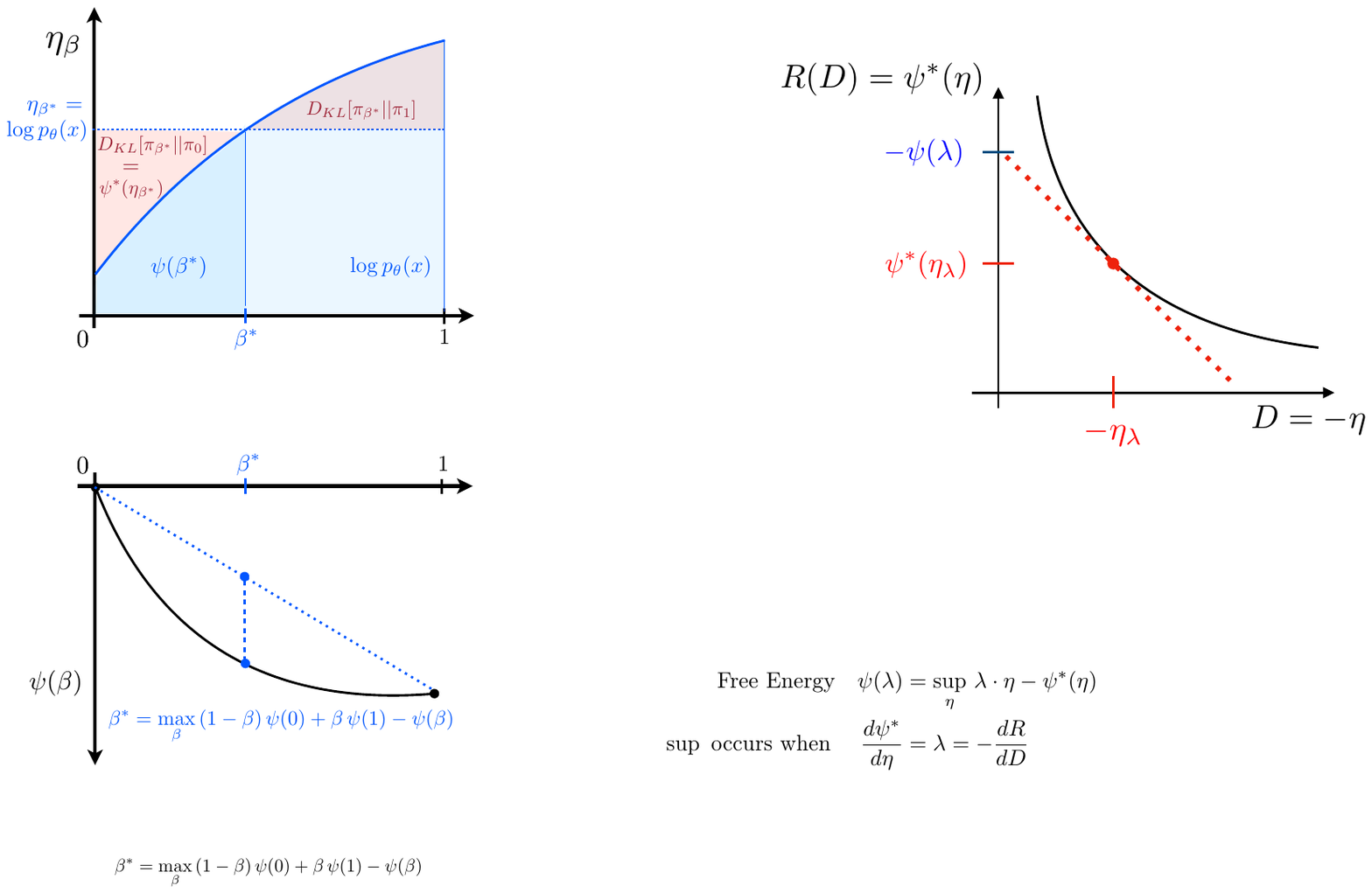}
%\includegraphics[width=\textwidth]{figs/alpha_star.pdf}
%\vspace*{.175cm}
\vspace*{.525cm}
\caption{Chernoff point on $\psi(\beta)$} \label{fig:chernoff_logzb} 
\end{minipage}
\end{figure*}

\paragraph{Chernoff Information}
While each choice of $\beta_{\eta}$ determines a likelihood ratio test and error region, how should we choose this parameter?   Regardless of the prior probabilities $p_0, p_1$ that we might assign to each hypothesis in a Bayesian setting, the Chernoff information provides the best achievable error exponent in the large sample limit (\cite{nielsen2013information}, \cite{cover2012elements} Ch. 11).
\begin{align}
C^{*} := - \min_{\beta} C(\beta) &= - \min \limits_{\beta} \, \log \int \pi_0(z)^{1-\beta} \pi_1(z)^{\beta} dz  \label{eq:chernoff} %\\
%&= - \max \limits_{\beta} \,  (1-\beta) \, \psi(0) + \beta \, \psi(1) - \psi \big( \beta \big) \label{eq:jgc}
\end{align}
%\max \limits_{\beta} \, (1-\beta) \, \psi(0) + \beta \, \psi(1) - \psi \big( \beta \big)
Notice that Chernoff information in \eqref{eq:chernoff} involves the log-partition function $C(\beta)$ for the geometric mixture between \textit{normalized} $\pi_0$ and $\pi_1$, %with $C^{*} = - \min_{\beta} C(\beta)$. 
whereas we have defined $\psi(\beta) = \log \int \tilde{\pi}_0(z)^{1-\beta} \tilde{\pi}_1(z)^{\beta} dz $ using \textit{unnormalized} $\tilde{\pi}_0$ and $\tilde{\pi}_1$.  Rewriting $C(\beta)$ using $\pi_0 = \tilde{\pi}_0 / Z_0$ and $\pi_1 = \tilde{\pi}_1 / Z_1$, we can pull out factors of $(1-\beta) \log Z_0$ and $\beta \log Z_1$ to obtain the relation $C(\beta) = \psi(\beta) - (1-\beta) \log Z_0 - \beta \log Z_1$.   The Chernoff information can thus be written using the Jensen gap $\mathcal{J}_{\psi, \{ 1-\beta, \beta \}, \{ 0,1 \} }$ from \eqref{eq:jg00} %negative gap in Jensen's inequality for $\psi(\beta)$,
% with $\psi(\beta) = \log \int \tilde{\pi}_0(z)^{1-\beta} \tilde{\pi}_1(z)^{\beta} dz $.    Rewriting $C(\beta)$ using $\pi_0 = \tilde{\pi}_0 / Z_0$ and $\pi_1 = \tilde{\pi}_1 / Z_1$, we can pull out the normalization factors $(1-\beta) \log Z_0$ and $\beta \log Z_1$ to relate $C(\beta)$ and $\psi(\beta)$ via a gap in Jensen's inequality  
\begin{align}
C^{*} := - \min_{\beta} C(\beta) 
&= - \max \limits_{\beta} \,  (1-\beta) \, \psi(0) + \beta \, \psi(1) - \psi \big( \beta \big) \, . \label{eq:jgc} %&= - \min \limits_{\beta} \, \log \int \pi_0(z)^{1-\beta} \pi_1(z)^{\beta} dz  \label{eq:chernoff} \\
\end{align}
%$\psi(\beta) = \log \int \tilde{\pi}_0(z)^{1-\beta} \tilde{\pi}_1(z)^{\beta} dz $.

%To rewrite \eqref{eq:chernoff} as a gap in Jensen's inequality \eqref{eq:jgc}, we have expanded $\pi_0 = \tilde{\pi}_0 / Z_0$ and $\pi_1 = \tilde{\pi}_1 / Z_1$, with $\psi(\beta) = \log \int \tilde{\pi}_0(z)^{1-\beta} \tilde{\pi}_1(z)^{\beta} dz $. 

%and pulling the $log$ normalizing constants out from $\pi_0 = \tilde{\pi}_0 / Z_0$ and $\pi_1 = \tilde{\pi}_1 / Z_1$ , we can 
% Writing unnormalized $\pi_0 = \tilde{\pi}_0 / Z_0$ and $\pi_1 = \tilde{\pi}_1 / Z_1$ in the likelihood ratio family \eqref{eq:lkd_ratio_fam}, with $\psi(\beta) = \log \int \tilde{\pi}_0(z)^{1-\beta} \tilde{\pi}_1(z)^{\beta} dz $, we can rewrite \eqref{eq:chernoff} as a gap in Jensen's inequality
% \begin{align}
%     C^{*} &= - \max \limits_{\beta} \,  (1-\beta) \, \psi(0) + \beta \, \psi(1) - \psi \big( \beta \big) \, . \label{eq:jgc}
% \end{align}

The optimum over $\beta$, or $\beta^{*}$,  is denoted the \textit{Chernoff point} \cite{nielsen2013information}.  
In App. \ref{app:max_beta}, we derive the moment-matching condition
\begin{align}
    \eta_{\beta^{*}} = \frac{\psi(\beta_1) - \psi(\beta_0) }{ \beta_1 - \beta_0} \label{eq:chern}
\end{align}
which holds between arbitrary $\beta_0, \beta_1$ and implies $\eta_{\beta^{*}} = \psi(1)-\psi(0)$ for $\beta_0=0, \beta_1 =1$.
%At this optimum, denoted the Chernoff point \cite{nielsen2013information}, 
At this critical point, the \textsc{kl} divergence to the endpoints is equal, as shown in App. \ref{app:equal_kl}
\begin{align}
D_{\KL}[\pi_{\beta^{*}}(z) || \pi_0(z)] = D_{\KL}[\pi_{\beta^{*}}(z) || \pi_1(z)] \label{eq:equal_kl} \, .
\end{align}
%For normalized $\pi_0, \pi_1$ with $\psi(0) = \psi(1) = 0$ in \eqref{eq:chern}, the optimal decision rule is given by a threshold of $\eta_{\beta^{*}} = \mathbb{E}_{\pi_{\beta^{*}}} \log \frac{\pi_1(\z)}{\pi_0(\z)} = 0 $.   %With the unnormalized likelihood ratio $\log \tilde{\pi}_1(\z) / \pi_0(\z)$, we show in the next section that the Chernoff poinrt 

\paragraph{Chernoff Point on the TVO Integrand}
For the unnormalized likelihood ratio $\log \tilde{\pi}_1(\z) / \pi_0(\z)$, we can interpret the Chernoff point using  thermodynamic integration bounds \eqref{eq:tvo_lb_ub2} %from the \gls{TVO} \cite{masrani2019thermodynamic, brekelmans2020all}
%Recalling the thermodynamic integration bounds \eqref{eq:tvo_lb_ub2} used in \gls{TVO} \cite{masrani2019thermodynamic, brekelmans2020all}
\begin{align}
  \hspace*{-.2cm} \sum \limits_{t=0}^{T-1} (\beta_{t+1} - \beta_{t}) \cdot \mathbb{E}_{\pi_{\beta_{t}}} \big[ \log \frac{p_{\theta}(\vx,\vz)}{q(\vz|\vx)} \big] \, \leq \, \log Z_1 \, \leq \,\sum \limits_{t=0}^{T-1} (\beta_{t+1} - \beta_{t}) \cdot \mathbb{E}_{\pi_{\beta_{t+1}}} \big[ \log \frac{p_{\theta}(\vx,\vz)}{q(\vz|\vx)} \big] \label{eq:tvo_lb_ub2} \, .
    % \log \frac{1}{K} \sum \limits_{k=1}^K \wkt := \mathcal{L}^{\gls{TVI}}_{\gls{IWAE}_{\textsc{lb}}}\label{eq:ti_iwae_lb}
\end{align}
% \begin{align}
%   \sum \limits_{t=0}^{T-1} (\beta_{t+1} - \beta_{t}) \cdot \mathbb{E}_{\pi_{\beta_{t}}} \big[ \log \frac{\tilde{\pi}_1(\vx, \vz)}{\tilde{\pi}_0(\vz)} \big] \, \leq \, \log Z_1 \, \leq \,\sum \limits_{t=0}^{T-1} (\beta_{t+1} - \beta_{t}) \cdot \mathbb{E}_{\pi_{\beta_{t+1}}} \big[ \log \frac{\tilde{\pi}_1(\vx, \vz)}{\tilde{\pi}_0(\vz)} \big] %\label{eq:tvo_lb_ub3} \, ,
%     % \log \frac{1}{K} \sum \limits_{k=1}^K \wkt := \mathcal{L}^{\gls{TVI}}_{\gls{IWAE}_{\textsc{lb}}}\label{eq:ti_iwae_lb}
% \end{align}
With $\pi_0(\vz)= q(\vz|\vx)$ as in \gls{TVO} \cite{masrani2019thermodynamic, brekelmans2020all}, we note that the integrand at $\beta_0=0$ corresponds to the familiar \gls{ELBO}, $\mathbb{E}_{\pi_{0}} \big[ \log \frac{\tilde{\pi}_1(\vx, \vz)}{\pi_0(\vz|\vx)} \big] = \log Z_1(\vx) - D_{\KL}[q(\vz|\vx)|| p_{\theta}(\vz|\vx)]$.  %This reduces  for $\pi_0(\vz)= q(\vz|\vx)$ and lower bounds the marginal likelihood $\log \px$.  
Similarly, at $\beta_1=1$, the integrand $\mathbb{E}_{\pi_{1}}[\cdot] = \log Z_1(\vx) + D_{\KL}[p_{\theta}(\vz|\vx)|| q(\vz|\vx)]$ is an upper bound.  

Since $\psi(1) = \log \px$ and $\psi(0) =0$, the condition for the Chernoff point in \eqref{eq:chern} corresponds to \begin{align}
    \eta_{\beta^{*}} = \mathbb{E}_{\pi_{\beta^{*}}}\bigg[\log \frac{p_{\theta}(\vx,\vz)}{q(\vz|\vx)} \bigg] = \log \px \, ,
\end{align}%$\eta_{\beta^{*}} = \mathbb{E}_{\pi_{\beta^{*}}}[\cdot] = \log \px$,
 or the point after which the expected likelihood ratio switches from an lower bound to an upper bound.  We visualize this in Fig. \ref{fig:chernoff}, with $\eta_{\beta^{*}}$, as a point on the y-axis, equal to the area under the curve, $\log \px$.  Note that the red shaded regions correspond to the \kld divergence from $\pi_{\beta^{*}}$ to each endpoint (see \cite{brekelmans2020all}), and will have equal area due to \cref{eq:equal_kl}.
% \begin{align}
%     \psi^{*}(\log \px) = \sup \limits_{\beta} \beta \cdot \log \px - \psi(\beta) = \beta^{*} \cdot \log \px - \psi(\beta^{*})
% \end{align}

\section{Conclusion}
We have presented likelihood ratio exponential families as a common framework for understanding \gls{TVO}, \gls{RD}, \gls{IB}, and \textsc{rdc} optimizations in terms of conjugate duality and hypothesis testing.   These insights may be useful for improving mutual information estimators which leverage intermediate distributions \cite{rhodes2020telescoping}, learn a binary classifier distinguishing samples from $\tilde{\pi}_0, \tilde{\pi}_1$ \cite{tsai2020neural,liao2020demi}, or involve a neural network `critic' whose optimal function output is the true likelihood ratio %\cite{ %belghazi2018mine, 
\cite{poole2019variational}.%, %song2019understanding}.   

While it is natural to introduce additional sufficient statistics from the exponential family perspective, thermodynamic integration bounds and hypothesis testing interpretations remain to be clarified in higher dimensions as in \textsc{rdc}.  Further exploring the approach of \citet{gao2020free}, for evolving model parameters and Lagrange multipliers along the equilibrium surface of solutions to \cref{eq:legendre}, is an exciting future direction.  Beyond the applications shown in \cite{gao2020free}, this could lead to replacing heuristics such as \textsc{kl} annealing in $\beta$-\gls{VAE} with more principled dynamics for $\beta$ over the course of optimization.

%\section*{Acknowledgements}
%RB thanks Alireza Makhzani and Vaden Masrani for helpful comments and discussion.
%\input{sections/new_main/3_rdc}

% Using equilibrium in the Legendre duality expression in \eqref{eq:legendre} as a condition for incrementing $\beta$ in optimization heuristics \textsc{kl}

% As several existing variational bounds on mutual information  \cite{belghazi2018mine, poole2019variational, song2019understanding} involve a neural network `critic' whose optimal output is the true likelihood ratio, our approach may lend insights for improving 

%   Finally, several existing variational bounds on mutual information  \cite{belghazi2018mine, poole2019variational, song2019understanding} involve a neural network `critic' whose optimal output is the true likelihood ratio.  Although our approach appears to rely on tractable joint and marginal densities, it lend insights for improving mutual information estimators which leverage intermediate distributions \cite{rhodes2020telescoping} or  binary classification of samples \cite{liao2020demi}.  

% \clearpage
% \input{sections/main/intro.tex}
% \input{sections/main/likelihood_ratio.tex}
% %\input{sections/main/variational_rep.tex}
% \input{sections/main/hypothesis_testing}
% \input{sections/main/examples}
% %\input{sections/main/rdc.tex}
% %\input{sections/main/background.tex}

% %\input{sections/method.tex}
% %\input{sections/related.tex}

% %\input{sections/app/q}

% %\input{sections/alpha_proj.tex}

%\clearpage
\bibliographystyle{unsrtnat}
\bibliography{ref}
\clearpage
\appendix 
\section{Conjugate as a KL Divergence} \label{app:conj}
When considering an exponential family of the form
\begin{align} %-d\textbf{}(\vx, \vz)
    \pi_{\beta}(\vz)  = \pi_0(\vz) \exp \{ \beta  \cdot  \phi(\vz) \,   - \psi( \beta) \} \, . \label{eq:lkd_ratio_famA}
\end{align}
we show that $\psi^{*}(\eta)$ takes the form of a \textsc{kl} divergence when considering a base measure $\pi_0(\vz)$. %q(\vz|\vx)$ as in the \gls{TVO}.
\begin{align}
    \psi^{*}(\eta) &= \sup \limits_{\beta} \beta \cdot \eta - \psi(\beta) \label{eq:dual_opt} \\
    &= \beta_{\eta} \cdot \eta - \psi(\beta_{\eta}) \nonumber \\
    &=  \mathbb{E}_{\pi_{\beta_{\eta}}} [\beta_{\eta} \cdot \phi(\vz)] - \psi(\beta_{\eta}) \nonumber \\
    &= \mathbb{E}_{\pi_{\beta_{\eta}}}  [\beta_{\eta} \cdot \phi(\vz)]  - \psi(\beta_{\eta}) \pm \mathbb{E}_{\pi_{\beta_{\eta}}} [\log \pi_0(\vz) ] \nonumber \\
    &= \mathbb{E}_{\pi_{\beta_{\eta}}} [ \log \pi_{\beta_{\eta}(\vz)} - \log \pi_0(\vz)] \nonumber \\
    &= D_{KL}[\pi_{\beta_{\eta}}(\vz) || \, \pi_0(\vz)] \label{eq:conjugate_kl2}
    %&= \mathbb{E}_{\pi_{\beta_{\eta}}}  \, \log \frac{p(\vx, \vz)^{\beta_{\eta}} }{q(\vz|\vx)^ {\beta_{\eta}}} - \psi(\beta_{\eta}) \pm \mathbb{E}_{\pi_{\beta_{\eta}}} [\log \pi_0(\vz) ] \nonumber \\
    %&= \mathbb{E}_{\pi_{\beta_{\eta}}}  \log \frac{q(\vz|\vx)^{1- \beta_{\eta}} p(\vx, \vz)^{\beta_{\eta}}}{Z_{\beta_{\eta}} \, \cdot \, q(\vz|\vx)} \label{eq:add_subtract} \\
    %&= D_{KL}[\pi_{\beta_{\eta}}(\vz|\vx) || \, q(\vz|\vx)] \label{eq:conjugate_kl}
\end{align}
 where we have added and subtracted a factor of $\mathbb{E}_{\pi_{\beta_{\eta}}} \log \pi_0(\vz)$ in the fourth line.  If $\pi_0(\vz)$ is constant with respect to $\vz$ using, for example, the uniform measure, then $D_{KL}[\pi_{\beta_{\eta}}(\vz) || \, \pi_0(\vz)]$ reduces to the familiar definition of the conjugate $\psi^{*}(\eta)$ as the negative entropy $ \mathbb{E}_{\pi_{\beta_{\eta}}} \log \pi_{\beta_{\eta}} (\vz)$ \cite{wainwrightjordan}.

\section{KL Divergence as a Bregman Divergence}\label{app:breg_kl}
For an exponential family with partition function $\psi(\beta)$ and sufficient statistics $\phi(z)$ over a random variable $\vz$, the Bregman divergence $D_{\psi}$ corresponds to a KL divergence.  Recalling that $\nabla_{\beta} \psi(\beta)= \eta_{\beta} = \mathbb{E}_{\pi_\beta}[\phi(z)]$, we simplify the definition of the Bregman divergence  to obtain
\begin{align}
D_{\psi}[\beta : \beta^{\prime}] &=   \psi(\beta) - \psi(\beta^{\prime}) - \langle \beta - \beta^{\prime}, \nabla \psi(\beta^{\prime}) \rangle \nonumber \\
&= \psi(\beta) - \psi(\beta^{\prime}) -  \beta \cdot \eta_{\beta^{\prime}} + \beta^{\prime} \cdot \eta_{\beta^{\prime}} \nonumber\\
&= \psi(\beta) - \psi(\beta^{\prime}) -  \mathbb{E}_{\pi_{\beta^{\prime}}}[\beta \cdot \phi(z)] + \mathbb{E}_{\pi_{\beta^{\prime}}}[\beta^{\prime} \cdot \phi(z)] \nonumber \\[1.5ex]
&= \underbrace{\mathbb{E}_{\pi_{\beta^{\prime}}}\big[\beta^{\prime} \cdot \phi(z) - \psi(\beta^{\prime})\big] + \mathbb{E}_{q}[\pi_0(\vz)]}_{ \log \pi_{\beta^{\prime}}(\vz) } - \underbrace{\mathbb{E}_{\pi_{\beta^{\prime}}}\big[\beta \cdot \phi(z) - \psi(\beta)\big] - \mathbb{E}_{q}[\pi_0(\vz)]}_{ \log \pi_{\beta}(\vz) } \nonumber \\
&= \mathbb{E}_{\pi_{\beta^{\prime}}} \big[ \log \pi_{\beta^{\prime}}(\vz) - \log\pi_{\beta}(\vz) \big] \nonumber \\
&= D_{KL}[\pi_{\beta^{\prime}}(\vz) || \pi_{\beta}(\vz)] \label{eq:breg_kl}
\end{align}
where we have added and subtracted terms involving the base measure $\pi_0(\vz)$, and used the definition of our exponential family from \eqref{eq:lkd_ratio_famA}.  The Bregman divergence $D_{\psi}$ is thus equal to the KL divergence with arguments reversed.

\subsection{Canonical Divergence}\label{app:canonical}  
We can also show that the Bregman divergences $D_{\psi}, D_{\psi^{*}}$ are equivalent up to reordering of the arguments
\begin{align}
    D_{\psi}[\beta : \beta^{\prime}] = D_{\psi^{*}}[\eta^{\prime} : \eta]
\end{align}
where we abbreviate $\eta^{\prime} = \eta_{\beta^{\prime}}$ and $\eta = \eta_{\beta}$.  The conjugacy relationships 
\begin{align}
\psi^{*}(\eta^{\prime}) = \beta^{\prime} \cdot \eta^{\prime} - \psi(\beta^{\prime})  \qquad
\psi(\beta) = \beta \cdot \eta - \psi^{*}(\eta)
\end{align}
can be used to translate between these dual divergences. 
\begin{align}
    D_{\psi}[\beta : \beta^{\prime}] &=   \psi(\beta) - \psi(\beta^{\prime}) - \langle \beta - \beta^{\prime}, \nabla \psi(\beta^{\prime}) \rangle \\
    &= \psi(\beta) - \psi(\beta^{\prime}) -  \beta \cdot \eta^{\prime} + \beta^{\prime} \cdot \eta^{\prime} \nonumber\\ 
    &= \psi(\beta) + \psi^{*}(\eta^{\prime}) -  \beta \cdot \eta^{\prime} \label{eq:canonical} \\
 &= \psi^{*}(\eta^{\prime}) - \psi^{*}(\eta) +  \beta \cdot \eta -  \beta \cdot \eta^{\prime}  \nonumber \\
  &= \psi^{*}(\eta^{\prime}) - \psi^{*}(\eta) - \langle    \eta^{\prime} - \eta  , \nabla \psi^{*}(\eta) \rangle \nonumber\\
 &= D_{\psi^{*}}[ \eta^{\prime}:\eta ] \nonumber
\end{align}
The intermediate expression \eqref{eq:canonical} is known as the canonical form of the divergence \cite{amari2016information}   
\begin{align}
      \psi(\beta) + \psi^{*}(\eta^{\prime}) -  \beta \cdot \eta^{\prime} = D_{\psi}[\beta : \beta^{\prime}] = D_{\psi^{*}}[ \eta^{\prime}:\eta ] 
\end{align}
Comparing with the expression for Legendre duality in \eqref{eq:legendre}, note that the correspondence between $\beta$ and $\eta$ implies that the divergence vanishes, since both parameterizations refer to the same distribution
\begin{align}
 \psi(\beta_{\eta}) + \psi^{*}(\eta_{\beta}) - \beta_{\eta} \cdot \eta_{\beta} = 0  \implies D_{\psi}[\beta_{\eta} : \beta_{\eta}] = D_{\psi^{*}}[ \eta_{\beta} :\eta_{\beta} ] = 0
\end{align}

%  To see that the Bregman divergence $D_{\psi^{*}}$ or $D_{\psi}$ correspond to the KL divergence,  recall that 
%  $\nabla \psi^{*}(\eta^{\prime}) = \beta^{\prime}$ due to conjugacy \cite{wainwrightjordan}.   We can then write
%  consider 
%  \begin{align}
%      D_{\psi^{*}}[\eta : \eta^{\prime}] &= \psi^{*}(\eta) - \psi^{*}(\eta^{\prime}) - \langle \eta - \eta^{\prime}, \nabla \psi^{*}(\eta^{\prime}) \rangle \\
%      &= D_{KL}[\pi_{\beta_{\eta}}(\vz) || \, \pi_0(\vz)] - D_{KL}[\pi_{\beta_{\eta^{\prime}}}(\vz) || \, \pi_0(\vz)] - \langle \eta - \eta^{\prime},  \beta^{\prime} \rangle \\
%      &= \mathbb{E}_{\pi_{\beta_{\eta}}}\big[\log \cancel{\frac{\pi_0(\vz)}{\pi_0(\vz)}} + \cancel{\log \exp} \{\beta \cdot \phi(\vz) - \psi(\beta) \}\big] \\
%      &\phantom{===}- \mathbb{E}_{\pi_{\beta_{\eta^{\prime}}}}\big[\cancel{\frac{\pi_0(\vz)}{\pi_0(\vz)}} + \cancel{\log \exp} \{\beta^{\prime} \cdot \phi(\vz) - \psi(\beta^{\prime}) \} \big]  -    \beta^{\prime} \cdot \eta  +   \beta^{\prime} \cdot \eta^{\prime} \\[1.25ex]
%      &= \beta \cdot \eta - \psi(\beta) - \beta^{\prime} \cdot \eta^{\prime} + \psi(\beta^{\prime}) -    \beta^{\prime} \cdot \eta  +   \beta^{\prime} \cdot \eta^{\prime} \\
%      &= 
%  \end{align}

\section{Information Bottleneck as Rate-Distortion} \label{app:ib_pf}
The Information Bottleneck (IB) method \cite{tishby1999information} defines the `relevant information' in a representation, $I(Y:Z)$, via another variable of interest $Y$, often taken to be a label.   The IB objective then seeks a minimal encoding $Z$ which maintains a given level of predictive ability about the target.
\begin{align}
    \min \limits_{q(\vz|\vx)} ~~&\iq ~~\text{subj. to.} ~~I_q(Y;Z) \geq I_c  \label{eq:ib} %\\
\end{align}
where we let $I_q$ reflect the exact mutual information for the true data and label distributions $q(x) q(y|x)$ with a given encoding function $q(z|x)$.

When the desired information constraint equals the total information $I_c = I_q(X;Y)$ that the data source contains about the label, \eqref{eq:ib} corresponds to the problem of finding the minimal sufficient statistics $z$ for $y$ with respect to $x$.  The IB objective generalizes this optimization for smaller values of $I_c$.

Since $I_q(Y;Z)= H_q(Y) - H_q(Y|Z) = -\mathbb{E}_q \log q(y) + \mathbb{E}_q \log q(y|\vz)$, we can ignore the label entropy as a constant with respect to $z$.   While it may be difficult to obtain the true posterior $q(y|\vz)$ of the labels given latent variables , we can instead optimize a variational classifier $p(y|\vz)$.  This provides an lower bound on the mutual information since $D_{KL}[q(y|\vz)||p(y|\vz)] \geq 0$ and is also known as the `test channel' in rate-distortion theory (\cite{cover2012elements} Ch. 13).   Applying this inequality within the unconstrained \gls{IB} Lagrangian,
\begin{align}
    \mathcal{L}_{IB} &= \max \limits_{\beta} \min \limits_{q(\vz|\vx)} ~~\iq  - \beta \, \big ( -\mathbb{E}_q \log q(y) + \mathbb{E}_q \log q(y|\vz) - I_c \big ) \nonumber \\
    &\geq \max \limits_{\beta} \min \limits_{q(\vz|\vx), p(y|\vz)} ~~\iq  - \beta \, \big ( -\mathbb{E}_q \log q(y) + \mathbb{E}_q \log p(y|\vz) - I_c \big ) \nonumber \\
    &= \max \limits_{\beta} \min \limits_{q(\vz|\vx), p(y|\vz)} ~~\iq  - \beta \, \mathbb{E}_{q(y(\vx), \vz)}[ p(y|\vz) ] + \text{const} \label{eq:ib_lagr}
\end{align}
where $y(\vx)$ indicates the label of a given data point.

As shown in \citet{tishby1999information}, the Information Bottleneck is a special case of rate-distortion with 
\begin{align}
c(y(\vx), \vz) = D_{KL}[q(y|\vx) || q(y|\vz)] = \mathbb{E}_{q}[ q(y|\vx) ] - \mathbb{E}_{q}[ q(y|\vz) ]   \label{eq:rdlagr}
\end{align}
Comparing \eqref{eq:ib_lagr} with \eqref{eq:rdlagr}, note that $\mathbb{E}_{q}[q(y|\vx) ]$ is a constant, leaving the effective distortion measure as $ c(y(\vx), \vz) = - \mathbb{E}_{q(y(\vx)|\vz)}[ q(y|\vz) ] $.  If this quantity is intractable, we can instead define the distortion function using a variational $p(y|\vz)$ as above.

\section{Bregman Information and Jensen Gaps} \label{sec:jensen}
Imagine we are interested in minimizing the expected Bregman divergence $D_{f}$ to a single representative point, which may then be thought of as the optimal codeword for $R = 0$ in a rate-distortion scenario using a Bregman divergence distortion.  \citet{Banerjee2005} show that, regardless of the divergence, the minimizing point will be the mean with respect to a desired measure, and the expected divergence will be a gap in Jensen's inequality for the function $f$. %(see App. \ref{app:jensen_gaps})

\begin{theorem}[Bregman Information, \cite{Banerjee2005}]\label{thm:bregman} 
Let X be a random variable that takes values in $\mathcal{X} = \{\vx_i\}_{i=1}^{n} \subseteq \mathcal{S} \subseteq \mathbb{R}^d$ following a positive probability measure $\nu$ such that $\mathbb{E}_{\nu}[X] \in ri(\mathcal{S})$.  Given a Bregman divergence $D_{f} : \mathcal{S} \times ri(\mathcal{S}) \mapsto [0,\infty]$, the problem:
\begin{align*}
\mathcal{J}_{\nu, f}(s^{*}) = \min \limits_{s \in ri(\mathcal{S})} \mathbb{E}_{\nu}[ D_f(X, s) ]
\end{align*}
has a unique minimizer given by the mean $s^{*} = \mu = \mathbb{E}_{\nu}[X]$.  At this $\argmin$, $\mathcal{J}_{\nu, f}(s^{*})$ corresponds to a gap in Jensen's inequality for the convex function $f$ and expectations with respect to $\nu$.
\begin{align*}
   \mathbb{E}_{\nu}[ f(X) ] - f(\mathbb{E}_{\nu}[X]) = \mathcal{J}_{\nu, f}(s^{*}) 
\end{align*}
\end{theorem}
\begin{proof}
Consider a point $s$ and the mean $\mu = \mathbb{E}_{\nu}[X]$, both in $ri(\mathcal{S})$ so that $\mathcal{J}_{\nu, f}$ is well defined. 
\begin{align*}
    \mathcal{J}_{\nu, f}(s) - \mathcal{J}_{\nu, f}(\mu) &= \mathbb{E}_{\nu} D_{f}(\vx, s) -  \mathbb{E}_{\nu} D_{f}(\vx, \mu) \\
    &= \cancel{\mathbb{E}_{\nu} f(x)} - f(s) - \langle \mathbb{E}_{\nu} x - s, \nabla f(s) \rangle \\
     & \phantom{=} - \cancel{\mathbb{E}_{\nu} f(x)} + f(\mu) - \langle \cancel{\mathbb{E}_{\nu} x - \mu}, \nabla f(\mu) \rangle \\
    &= f(\mu) - f(s) - \langle \mu - s, \nabla f(s) \rangle \\
    &= D_f[\mu, s] \geq 0
\end{align*}
with equality only when $s = \mu =\mathbb{E}_{\nu} [X]$ if $f$ is strictly convex.  Then,
\begin{align*}
    \mathcal{J}_{\nu, f}(\mu) &= \mathbb{E}_{\nu}[ D_f(X, \mu) ] \\
    &= \mathbb{E}_{\nu}[f(X)] - f(\mu) - \underbrace{\mathbb{E}_{\nu} \langle x - \mu}_{0}, \nabla_x f(\mu) \rangle \\
    &= \mathbb{E}_{\nu}[f(X)] - f(\mu)
\end{align*}
which amounts to a gap in Jensen's inequality for the function $f$ and measure $\nu$.
\end{proof}
%For the Bregman Information considered here, $f(X)$ is the log partition function $\psi(\beta)$.  T

\section{Gap in Jensen's Inequality for \texorpdfstring{$\psi(\beta)$}{the Log Partition Function}}
In this section, we analyse the Bregman Information and Jensen gap associated with the log partition function of the likelihood ratio exponential family.  As in Sec. \ref{sec:vrep}, this corresponds to the variational representation of \citet{grosse2013annealing}, while maximizing the Jensen gap will lead to the Chernoff point.  We give proofs of intermediate results at the end of the section.

With $D_{\psi}$ as the divergence associated with the convex function $f(X)=\psi(\beta)$, we take the expected Bregman divergence using a convex combination ($\{1-\alpha, \alpha \}$) over arguments $\{ \beta_0, \beta_1\}$.   
\begin{align}   
  \mathcal{J}_{\psi, \alpha} &= \min \limits_{\beta^{\prime}}  \, \, (1-\alpha) D_{\psi}[\beta_0 : \beta^{\prime} ] + \alpha D_{\psi}[\beta_1: \beta^{\prime}] \label{eq:jg_min_primal}\\
  \mu^{(\alpha)}_{\beta} &= (1-\alpha) \, \beta_0 + \alpha \, \beta_1 %\\
\end{align}
where \cref{thm:bregman} shows that the minimizer $\mu^{(\alpha)}_{\beta}$ occurs at the mean of the arguments.  After simplifying, we can see that this corresponds to a Jensen's inequality for the convex function $\psi(\beta)$ %$\beta_{\alpha}^{*}$, %we see that Jensen's inequality   
% can use  \cref{thm:bregman} to recognize this as a gap in 
\begin{align}
  \mathcal{J}_{\psi, \alpha} &= (1-\alpha) \,  D_{\psi}[\, \beta_0 :  \mu^{(\alpha)}_{\beta} ] + \, \alpha \, D_{\psi}[ \, \beta_1: \mu^{(\alpha)}_{\beta}]\\
  &= (1-\alpha) \, \psi(\beta_0) +  \alpha \, \psi(\beta_1) - \psi(\mu^{(\alpha)}_{\beta}) \\
  &\phantom{= (1-\alpha)}\underbrace{- (1-\alpha)(\beta_0 - \mu^{(\alpha)}_{\beta}) \nabla \psi(\mu^{(\alpha)}_{\beta}) - \alpha (\beta_1 - \mu^{(\alpha)}_{\beta})\nabla \psi(\mu^{(\alpha)}_{\beta})}_{=(\mu^{(\alpha)}-\mu^{(\alpha)})\nabla \psi(\mu^{(\alpha)}) = 0}  \nonumber \\[1.25ex]
  &= (1-\alpha) \, \psi(\beta_0) + \alpha \, \psi(\beta_1) - \psi \big(  (1-\alpha) \beta_0 + \alpha \, \beta_1 \big) \label{eq:jg_primal}
\end{align}
For the case of $\beta_0 = 0$ and $\beta_1 = 1$, we see that the optimal parameter is simply $\mu^{(\alpha)}_{\beta}  = \alpha$ so that
\begin{align}
    \mathcal{J}_{\psi, \alpha} = (1-\alpha) \, \psi(0) + \alpha \, \psi(1) - \psi(\alpha) 
\end{align}   

\citet{nielsen2011renyi} demonstrate the following lemma, showing the relationship between the Jensen's gap and R\'enyi divergence within an exponential family.

\begin{lemma} The R\'enyi divergence of order $\alpha$ between two distributions (indexed by natural parameters $\beta_0$ and $\beta_1$) within an exponential family (with log partition function $\psi(\beta)$) has the form of a gap in Jensen's inequality $\mathcal{J}_{\psi, \{1-\alpha, \alpha\}, \{\beta_0, \beta_1 \} }$, abbreviated  $\mathcal{J}_{\psi, \alpha}$
\begin{align}
   (1-\alpha) D_{\alpha}[\pi_{\beta_0}(z) : \pi_{\beta_1}(z)] &= - \log \int \pi_{\beta_0}(z)^{1-\alpha} \pi_{\beta_1}(z)^{\alpha} dz \\ 
 &=(1-\alpha) \, \psi(\beta_0) + \alpha \, \psi(\beta_1) - \psi \big(  (1-\alpha) \beta_0 + \alpha \, \beta_1 \big) \\
 &=  \mathcal{J}_{\psi, \alpha}  \label{eq:jg_lemma}
\end{align}
%where $\psi(\beta)$ is the log-partition function of the exponential family.
\end{lemma}
\begin{proof}
The first equality follows from the definition of $D_{\alpha}$, and corresponds to the Chernoff coefficient in \eqref{eq:chernoff} or Jensen gap $\mathcal{J}_{\psi, \alpha}$.
We demonstrate the second equality for the likelihood ratio family in App. \ref{app:renyi_jg}, or see \citet{nielsen2011renyi} for the general case.
\end{proof}

%Note that $\mathcal{J}_{\psi, \alpha}$

\citet{van2014renyi} show that the scaled Renyi divergence $(1-\alpha) D_{\alpha}[\pi_{\beta_0} : \pi_{\beta_1}]$ is concave. 
%and corresponds to the gap in Jensen's inequality $\mathcal{J}_{\psi, \alpha}$ \cite{nielsen2011renyi}.   
For given endpoint distributions, we can thus seek to maximize $\mathcal{J}_{\psi, \alpha}$ as a function of $\alpha$.    

\begin{lemma} Maximizing the Jensen's gap $\mathcal{J}_{\psi, \alpha}$ obtained from arguments $\{\beta_0, \beta_1\}$ of the convex function $\psi(\beta)$, with respect to the choice of mixing weight $\{1-\alpha, \alpha \}$,
\begin{align}
    \argmax \limits_{\alpha} \, \mathcal{J}_{\psi, \alpha} =\argmax \limits_{\alpha} (1-\alpha) \, \psi(\beta_0) + \alpha \, \psi(\beta_1) - \psi \big(  (1-\alpha) \beta_0 + \alpha \, \beta_1 \big) \label{eq:argmax}
\end{align}
leads to the following condition
\begin{align}
 \eta_{\alpha^{*}} &= \frac{\psi(\beta_1) - \psi(\beta_0)}{\beta_1 - \beta_0} %= \log \px
\end{align}
\end{lemma}
\begin{proof} See App. \ref{app:max_beta} for proof. \end{proof}

For $\beta_0 = 0$ and $\beta_1 = 1$, this suggests that the expected sufficient statistics $\eta_{\alpha}$ (with natural parameter $\mu_{\beta}^{(\alpha)} = \alpha$) should match the marginal likelihood $\log \px$.
\begin{align}
    \eta_{\alpha} &= \frac{\psi(1) - \psi(0)}{1-0} = \log \px
\end{align}

\begin{lemma}
At the maximum in Eq. \eqref{eq:argmax}, consider the distribution $\pi_{\beta^{*}_{\alpha}}$ in the same exponential family, with natural parameter $\beta^{*}_{\alpha} = \mu_{\beta}^{(\alpha)} = (1-\alpha) \, \beta_0 + \alpha \, \beta_1$, the Bregman divergences to each endpoint $\{\beta_0, \beta_1 \}$ are the same.
\begin{align}
     D_{\psi}[\beta_0 : \beta_{{\alpha}^{*}} ] &= D_{\psi}[\beta_1 : \beta_{{\alpha}^{*}} ] %\\
\end{align} 
Since the Bregman divergence $D_{\psi}$ within an exponential family corresponds to the KL divergence, we can equivalently write
\begin{align}
     D_{KL}[\pi_{\beta_{{\alpha}^{*}}} : \pi_{\beta_0}] &= D_{KL}[\pi_{\beta_{{\alpha}^{*}}} : \pi_{\beta_1}] 
\end{align}
\end{lemma}
\begin{proof} 
See App. \ref{app:equal_kl} for Bregman divergence derivations.   See discussion around \eqref{eq:divergences} for the relationship between the Bregman and KL divergence. 
%that at this maximum, $\mathcal{J}_{\psi, \alpha^{*}} = D_{\psi}[\beta_0 : \beta_{\alpha}^{*} ] = D_{\psi}[\beta_1 : \beta_{\alpha}^{*} ]$, so that the divergences are equal.  In particular, we have
\end{proof}
% \begin{align*}
%      D_{KL}[\pi_{\beta_{{\alpha}^{*}}} : \pi_{\beta_0}] = D_{KL}[\pi_{\beta_{{\alpha}^{*}}} : \pi_{\beta_1}] 
% \end{align*}

\subsection{R\'enyi Divergence as a Jensen Gap}\label{app:renyi_jg}
%\subsection{R\'enyi Divergence as a Jensen Gap}\label{app:renyi}
We consider the R\'enyi $\alpha$ divergence between any two distributions $\pi_{\beta_1}$ and $\pi_{\beta_0}$ in our exponential family, so that $\pi_{\beta}(\vz|\vx) = \pi_0(\vz)^{1-\beta} \pi_1(\vz)^{\beta} / Z_{\beta}(\vx)$.  Noting that the scaling factor $\alpha - 1 \leq 0$, we proceed to show that the scaled divergence is equal to a gap in Jensen's inequality:
\begin{align}
    (1-\alpha) & D_{\alpha}[\pi_{\beta_1}(\vz) : \pi_{\beta_0}(\vz)] \nonumber \\
    &=  (1-\alpha) \frac{1}{\alpha -1} \log \int \pi_{\beta_0}^{1-\alpha} \pi_{\beta_1}^{\alpha} d\vz \nonumber \\
    &= - \log \int \big(\frac{\pi_{0}^{1-\beta_0} \pi_{1}^{\beta_0}}{Z_{\beta_0}}\big)^{1-\alpha} \big(\frac{\pi_{0}^{1-\beta_1} \pi_{1}^{\beta_1}}{Z_{\beta_1}}\big)^{\alpha} d\vz \nonumber \\
    &= - \bigg( \log \int \pi_{0}^{1- \beta_0 - \alpha + \alpha \beta_0  + \alpha - \alpha \beta_1} \pi_{1}^{\beta_0 - \alpha \beta_0 + \alpha \beta_1} d\vz - \big((1-\alpha) \log Z_{\beta_0} + \alpha \log Z_{\beta_1} \big) \bigg) \nonumber \\
    &=  - \bigg( \log \int \pi_{0}^{1 - [(1- \alpha) \beta_0 + \alpha \beta_1]} \pi_{1}^{(1- \alpha) \beta_0 + \alpha \beta_1} d\vz  - \big((1-\alpha) \log Z_{\beta_0} + \alpha \log Z_{\beta_1} \big) \bigg) \nonumber \\
    &=  (1-\alpha) \psi(\beta_0) + \alpha \psi(\beta_1)  - \psi \big((1-\alpha)\beta_0 + \alpha \beta_1 \big) \nonumber \\ %\label{eq:R\'enyi} \\
    &=  \mathcal{J}_{\alpha, \psi} \nonumber
\end{align}

\subsection{Chernoff Point and Maximizing the Jensen Gap}\label{app:max_beta}
In this section, we derive the optimal solution for the optimization defining the Chernoff information point.  In particular, we optimize the Bregman Information or Jensen gap $\mathcal{J}$
\begin{align}
    \mathcal{J}_{\psi, \alpha} &= \max \limits_{\alpha} 
    \, \min \limits_{\beta_r} \, \, (1-\alpha) \,  D_{\psi}[\, \beta_0 :  \beta_r ] + \, \alpha \, D_{\psi}[ \, \beta_1: \beta_r]\\
    &=\max \limits_{\alpha}  \, \,  (1-\alpha) \, \psi(\beta_0) + \, \alpha \, \psi(\beta_1) - \, \psi((1-\alpha) \beta_0 + \alpha \beta_1) 
\end{align}
 where we use arbitrary endpoints $\beta_0, \beta_1$ and mixing parameter $\alpha$ to highlight the arithmetic mean in the argument of the final term.  This will match \eqref{eq:jgc} when $\beta_0 = 0$, $\beta_1 = 1 $ and $t = \beta$.

Now, we can differentiate with respect to $\alpha$, letting  $\beta_{\alpha} = (1-\alpha) \beta_0 + \alpha \beta_1$.   We use the product rule and the identity $d \psi(\beta) / d\beta = \eta_{\beta}$  in the last term to obtain
\begin{align}
    \frac{d\mathcal{{J}}}{dt} = 0 &= -\psi(\beta_0) + \psi(\beta_1) - \eta_{\beta_\alpha} \cdot ( \beta_1 - \beta_0 ) \\
        \implies \quad \eta_{\beta_{\alpha}} &= \frac{\psi(\beta_1)-\psi(\beta_0) }{ \beta_1 - \beta_0}
\end{align}
where $\eta_{\beta_{\alpha}}$ indicates the expected sufficient statistics, or dual parameter, corresponding to the natural parameter $\beta_{\alpha} = (1-\alpha) \beta_0 + \alpha \beta_1$.
% When $\beta_0 = 0$ and $\beta_1 = 1$, we have $\beta_t = (1-t) \cdot 0 + t \cdot 1 = t$ so that we can treat $\beta$ as the mixing parameter.   With $\psi(1) = \log \px$ and $\psi(0) = 0$, we have
% %, so that the \textsc{lhs} specifies a condition on the moment parameter associated with a nat
% \begin{align}
%     \eta_{\beta_t} = \frac{\log \px - 0 }{1-0} = \log \px
% \end{align}

\subsection{Equal KL Divergences Derivation}\label{app:equal_kl}
We show that the KL divergences that constitute $\mathcal{J}_{\alpha, \psi}$ are equal at the critical point $\eta_{\alpha} = \frac{\psi(\beta_1)-\psi(\beta_0)}{\beta_1-\beta_0}$:

\begin{align*}
D_{\psi}[ \beta_0 : &\beta_{\alpha}] = \psi(\beta_0) - \psi(\beta_{\alpha}) - (\beta_0 - \beta_{\alpha}) \eta_{\alpha} \\
&\phantom{\beta_{\alpha}]} = \psi(\beta_0) - \psi(\beta_{\alpha}) + \frac{(\beta_{\alpha} - \beta_{0})}{\beta_1-\beta_0} (\psi(\beta_1) - \psi(\beta_0)) \\
& = \small \frac{1}{\beta_1 - \beta_0}\bigg( (\beta_1 - \beta_0) \psi(\beta_0) - (\beta_1 - \beta_0) \psi(\beta_{\alpha}) + 
(\beta_{\alpha} - \beta_{0}) \psi(\beta_1) - (\beta_{\alpha} - \beta_{0}) \psi(\beta_0) \bigg)\\
& = \frac{1}{\beta_1 - \beta_0}\bigg( (\beta_1 - \beta_{\alpha}) \psi(\beta_0) + 
(\beta_{\alpha} - \beta_{0}) \psi(\beta_1) - (\beta_1 - \beta_0) \psi(\beta_{\alpha})  \bigg)  \\
  & = \bigg( \frac{\beta_{1} - \beta_{\alpha}}{\beta_1 - \beta_0}  \psi(\beta_0) + \frac{\beta_{\alpha} - \beta_{0}}{\beta_1 - \beta_0} \psi(\beta_1) - \psi(\beta_{\alpha}) \bigg) 
  \end{align*}
\begin{align*}
D_{\psi}[ \beta_1 : &\beta_{\alpha}] = \psi(\beta_1) - \psi(\beta_{\alpha}) - (\beta_1 - \beta_{\alpha}) \eta_{\alpha} \\
&\phantom{\beta_{\alpha}]} = \psi(\beta_1) - \psi(\beta_{\alpha}) - \frac{(\beta_{1} - \beta_{\alpha})}{\beta_1-\beta_0} (\psi(\beta_1) - \psi(\beta_0)) \\
&= \small \frac{1}{\beta_1 - \beta_0}\bigg( (\beta_1 - \beta_0) \psi(\beta_1) - (\beta_1 - \beta_0) \psi(\beta_{\alpha}) - 
(\beta_{1} - \beta_{\alpha}) \psi(\beta_1) + (\beta_{1} - \beta_{\alpha}) \psi(\beta_0) \big)\\
&= \frac{1}{\beta_1 - \beta_0}\bigg( (\beta_{1} - \beta_{\alpha}) \psi(\beta_0) + (\beta_{\alpha} - \beta_{0})
 \psi(\beta_1) - (\beta_1 - \beta_0) \psi(\beta_{\alpha}) \bigg) \\
   &= \bigg( \frac{\beta_{1} - \beta_{\alpha}}{\beta_1 - \beta_0}  \psi(\beta_0) + \frac{\beta_{\alpha} - \beta_{0}}{\beta_1 - \beta_0} \psi(\beta_1) - \psi(\beta_{\alpha}) \bigg)
% &= \frac{1}{\beta_1 - \beta_0}\bigg( (\beta_1 - \beta_{\alpha}) \psi(\beta_0) + 
% (\beta_{\alpha} - \beta_{0}) \psi(\beta_1) - (\beta_1 - \beta_0) \psi(\beta_{\alpha})  \bigg)
  \end{align*}
  
  We have shown that the two divergences are equal when our condition on $\eta_{\alpha}$ holds.   Further, observe that each divergence amounts to a Jensen gap $\mathcal{J}_{\alpha, \psi}$ with $\alpha = \frac{\beta_{\alpha} - \beta_{0}}{\beta_1 - \beta_0}$:
%   \begin{align*}
%   D_{\psi}[ \beta_0 : \beta_{\alpha}] &= D_{\psi}[ \beta_1 : \beta_{\alpha}] \\ 
%   &= \bigg( \frac{\beta_{1} - \beta_{\alpha}}{\beta_1 - \beta_0}  \psi(\beta_0) + \frac{\beta_{\alpha} - \beta_{0}}{\beta_1 - \beta_0} \psi(\beta_1) - \psi(\beta_{\alpha}) \bigg)
%   \end{align*}
 This is more apparent for $\beta_0 = 0$ and $\beta_1 = 1$, where this simplifies using $\alpha = \frac{\beta_{\alpha} - \beta_{0}}{\beta_1 - \beta_0} = \beta_{\alpha}$
  \begin{align*}
  D_{\psi}[ \beta_0 : \beta_{\alpha}] = D_{\psi}[ \beta_1 : \beta_{\alpha}] &= \frac{(1 - \beta_{\alpha})}{1-0} \psi(0) + \frac{\beta_{\alpha}-0}{1-0} \psi(1) - \psi(\beta_{\alpha}) \\[1.25ex]
%   &= (1 - \beta_{\alpha}) \cdot 0 +  \beta_{\alpha} \log p(\vx) \\
%   &\phantom{=} - \beta_{\alpha} \log p(\vx) + (1-\beta_{\alpha}) D_{\beta_{\alpha}}[  \pi_1(\vz|\vx) :  \pi_0(\vz|\vx)] \\[1.5ex]
  &= (1-\beta_{\alpha}) D_{\beta_{\alpha}}[ \pi_1(\vz|\vx) : \pi_0(\vz|\vx)]  \, , %\beta_{\alpha}  \psi(1) - \psi(\beta_{\alpha})
  \end{align*}
  where, in the last line, we use the fact that the scaled Rényi divergence is a Jensen gap from App. \ref{app:renyi_jg}.
  %although we often only have access to the unnormalized density $\tilde{\pi_1}(\vz|\vx)$  (e.g.$\pi_1 (\vz|\vx) = p(\vz|\vx)$), and must deal with the $\beta_{\alpha} \log p(\vx)$ terms in some way.

\subsection{Dual Jensen Gap using \texorpdfstring{$\psi^{*}(\eta)$}{ the Conjugate Function} }
Note that we could also construct a Jensen gap from the dual divergence $\psi^{*}(\eta_{\beta}) = D_{KL}[\pi_{\beta} || \pi_0]$, with $D_{\psi^{*}}[\eta: \eta^{\prime}] = D_{KL}[\pi_{\eta} || \pi_{\eta^{\prime}} ]$
%, with $\mathcal{J}_{\psi^{*}, 1-\lambda} &= over $\eta_0$ and $\eta_1$ with mixing parameter $1-\lambda$ used for later convenience.
% \begin{align}
%   \mathcal{J}_{\psi^{*}, 1-\lambda} &= \, \min \limits_{\eta^{\prime}}  \, \lambda D_{\psi^{*}}[\eta_0 : \eta^{\prime} ] + (1-\lambda) D_{\psi^{*}}[\eta_1: \beta^{\prime}] \label{eq:jg_min_dual} \\
%   \eta_{\lambda}^{*} &= \, \lambda \, \eta_0 + (1-\lambda) \, \eta_1 \nonumber \\
%   \end{align}\citet{nielsen2019jensen}
%   At this optimum, we obtain a gap in Jensen's inequality
  \begin{align}
  \mathcal{J}_{\psi^{*}, \lambda}   &= \lambda \, D_{KL}[\pi_{\eta_0} : \pi_{\mu_\eta^{(\lambda)}} ] + (1-\lambda) D_{KL}[\pi_{\eta_1} : \pi_{\mu_\eta^{(\lambda)}}] \label{eq:kl_min_dual} \\
  &=  \lambda \, D_{\psi^{*}}[\eta_0 : \mu_\eta^{(\lambda)}] + (1-\lambda) D_{\psi^{*}}[\eta_1: \mu_\eta^{(\lambda)}] \label{eq:jg_min_dual} \\
  &=  \, \lambda \, \psi^{*}(\eta_0) + (1-\lambda) \, \psi^{*}(\eta_1) - \psi^{*} \big(\lambda \, \eta_0 + (1-\lambda)  \, \eta_1 \big) \label{eq:jg_dual}
\end{align}
where $\mu_\eta^{(\lambda)} = \, \lambda \, \eta_0 + (1-\lambda) \, \eta_1$.  This matches the geometric Jensen-Shannon divergence of \citet{nielsen2019jensen} or \cite{deasy2020constraining}, whereas $\mathcal{J}_{\psi, \alpha}$ in \eqref{eq:jg_primal} is referred to as the dual version.
Taking the maximum over $\lambda$,
\begin{align*}
\lambda^{*} &= \argmax \limits_{\lambda} \mathcal{J}_{\psi^{*}, \lambda} \quad \\ %\implies \quad 
&= \frac{\psi^{*}(\eta_1) - \psi^{*}(\eta_0)}{\eta_1 - \eta_0}
\end{align*}
When $\pi_{\eta_0} = \pi_0$ corresponds to the base distribution, $\psi^{*}(\eta_0) = D_{KL}[\pi_{\beta_{\eta_0}} || \pi_0] = 0$.
\begin{align*}
    \lambda^{*} = \frac{D_{KL}[\pi_1 || \pi_0]}{D_{KL}[\pi_1 || \pi_0] + D_{KL}[\pi_0 || \pi_1] }
\end{align*}
% However, in more generality with $\pi_{\eta_0} \neq \pi_0$ note that $\eta_1 - \eta_0 = \frac{1}{\beta_1 - \beta_0} \big ( D_{KL}[\pi_{\eta_1} || \pi_{\eta_0}] + D_{KL}[\pi_{\eta_1} || \pi_{\eta_0}] \big) $ \cite{brekelmans2019tvo}.  
% \begin{align*}
%     \lambda^{*} &= \frac{\psi^{*}(\eta_1) - \psi^{*}(\eta_0)}{\frac{1}{\beta_1 - \beta_0} \big ( D_{KL}[\pi_{\eta_1} || \pi_{\eta_0}] + D_{KL}[\pi_{\eta_1} || \pi_{\eta_0}] \big)} \\
%     &= \frac{\big(\psi^{*}(\eta_1) - \psi^{*}(\eta_0)\big) \big( \beta_1 - \beta_0 \big) }{D_{KL}[\pi_{\eta_1} || \pi_{\eta_0}] + D_{KL}[\pi_{\eta_1} || \pi_{\eta_0}] }
% \end{align*}

% We can evaluate the condition of $\eta_{\lambda^{*}}$ using this $\argmax$:
% \begin{align*}
%     \eta_{\lambda}^{*} &= \lambda \eta_0 + (1-\lambda) \eta_1 \\
%     &= \frac{\psi^{*}(\eta_1) - \psi^{*}(\eta_0)}{\eta_1 - \eta_0} \cdot \eta_0 + (1- \frac{\psi^{*}(\eta_1) - \psi^{*}(\eta_0)}{\eta_1 - \eta_0}) \cdot \eta_1 \\ 
%     &= \eta_1 - \frac{\big( \eta_1-\eta_0 \big)\big(\psi^{*}(\eta_1) - \psi^{*}(\eta_0)\big) }{\eta_1-\eta_0} \\
%     &= \eta_1 - \psi^{*}(\eta_1) + \psi^{*}(\eta_0)
% \end{align*}

% At this maximum, the Jensen gap takes the value of
% \begin{align*}
%  \mathcal{J}_{\psi^{*}, \lambda}  = \, \, (\frac{\psi^{*}(\eta_1) - \psi^{*}(\eta_0)}{\eta_1 - \eta_0} ) \, \psi^*(\eta_0) + (1-\frac{\psi^{*}(\eta_1) - \psi^{*}(\eta_0)}{\eta_1 - \eta_0}) \psi^*(\eta_1) - \psi^{*}(\eta_{\lambda^{*}})
% \end{align*}

% \input{sections/app/appendix}

%\clearpage
% \input{sections/app/displaced_eq}
%\input{sections/app/noneq}

\end{document}